\documentclass[10pt]{article}
\usepackage[preprint]{rlj}
\usepackage{amssymb}
\usepackage{mathtools}
\usepackage{graphicx}
\usepackage{booktabs}
\usepackage{algorithm,algpseudocode}
\usepackage{amsthm}
\usepackage{amsmath}
\usepackage{enumitem}
\newtheorem{theorem}{Theorem}[section]
\newtheorem{lemma}[theorem]{Lemma}
\newtheorem{corollary}[theorem]{Corollary}
\newtheorem{proposition}[theorem]{Proposition}
\newtheorem{definition}[theorem]{Definition}
\newtheorem{assumption}[theorem]{Assumption}
\newtheorem{remark}[theorem]{Remark}
\usepackage{xcolor}

\newcommand{\AlgBorda}{\texttt{SA-MiDEX}}
\usepackage{dsfont}
\newcommand{\ind}{\mathds{1}}

\title{
Best-of-Both-Worlds Multi-Dueling Bandits: Unified Algorithms for Stochastic and Adversarial Preferences under Condorcet and Borda Objectives}
\setrunningtitle{Best-of-Both-Worlds Multi-Dueling Bandits}

\author{S Akash\textsuperscript{1}, Pratik Gajane\textsuperscript{2}, Jawar Singh\textsuperscript{1}}
\emails{}
\affiliations{
$^{1}$\textbf{Indian Institute of Technology Patna, India}\\
$^{2}$\textbf{University of Orléans, France}
}

\contribution{
    A black-box reduction (\texttt{MetaDueling}) from multi-dueling bandits to 
    dueling bandits that preserves expected pseudo-regret under the Condorcet 
    winner assumption.
}{
    The reduction builds on ideas from \citet{gajane2024adversarial} 
    and \citet{saha2018battle}.
}

\contribution{
    The first best-of-both-worlds algorithm for multi-dueling bandits in the 
    Condorcet setting, achieving $O(\sqrt{KT})$ adversarial pseudo-regret and 
    $O\!\left(\sum_{i \neq a^\star} \frac{\log T}{\Delta_i}\right)$ 
    instance-optimal pseudo-regret in stochastic environments, simultaneously 
    and without prior knowledge of the regime.
}{
    Obtained by combining our \texttt{MetaDueling} reduction with 
    \texttt{Versatile-DB} \citep{saha2022versatile}.
}

\contribution{
    A stochastic-and-adversarial algorithm (\AlgBorda) for 
    multi-dueling bandits with Borda objectives, achieving 
    $O\left(K^2 \log KT + K \log^2 T + \sum_{i: \Delta_i^{\mathrm{B}} > 0} \frac{K\log KT}{(\Delta_i^{\mathrm{B}})^2}\right)$ 
    regret in stochastic settings and $O\left(K \sqrt{T \log KT} + K^{1/3} T^{2/3} (\log K)^{1/3}\right)$ regret 
    against adversaries.
}{
    Adapts the stochastic-and-adversarial paradigm to multi-dueling bandits using 
    Successive Elimination for Borda estimation and an EXP3-style fallback \citep{gajane2024adversarial}.
}
\contribution{
Lower bounds establishing the optimality of our Condorcet guarantees in both regimes, and confirming that our Borda results are near-optimal (within a factor of $K$).
}{
 Derived via reduction from known dueling bandit lower bounds, and consistent with the established $\Omega(K^{1/3}T^{2/3})$ baseline for adv Borda objectives.
}


\summary{
\noindent
Multi-dueling bandits, where a learner selects $m \geq 2$ arms per round and observes only the winner, arise naturally in many applications including ranking and recommendation systems, yet a fundamental question has remained open: can a single algorithm perform optimally in both stochastic and adversarial environments, without knowing which regime it faces? We answer this affirmatively, providing the first best-of-both-worlds algorithms for multi-dueling bandits under both Condorcet and Borda objectives. For the Condorcet setting, we propose \texttt{MetaDueling}, a black-box 
reduction that converts any dueling bandit algorithm into a multi-dueling 
bandit algorithm by transforming multi-way winner feedback into an unbiased 
pairwise signal. Instantiating our reduction with \texttt{Versatile-DB} yields the first best-of-both-worlds algorithm for multi-dueling bandits: it achieves $O(\sqrt{KT})$ pseudo-regret against adversarial preferences and the 
instance-optimal $O\!\left(\sum_{i \neq a^\star} \frac{\log T}{\Delta_i}\right)$  pseudo-regret under stochastic preferences, both simultaneously and without prior knowledge of the regime. For the Borda setting, we propose \AlgBorda, a stochastic-and-adversarial algorithm that achieves $O\left(K^2 \log KT + K \log^2 T + \sum_{i: \Delta_i^{\mathrm{B}} > 0} \frac{K\log KT}{(\Delta_i^{\mathrm{B}})^2}\right)$ regret in stochastic environments and  $O\left(K \sqrt{T \log KT} + K^{1/3} T^{2/3} (\log K)^{1/3}\right)$ regret against  adversaries, without prior knowledge of the regime. We complement our upper bounds with matching lower bounds for the Condorcet setting. For the Borda setting, our upper bounds are near-optimal with respect to the lower bounds (within a factor of $K$) and match the best-known results in the literature.
}

\begin{document}

\makeCover

\maketitle
\begin{abstract}
Multi-dueling bandits, where a learner selects $m \geq 2$ arms per round and observes only the winner, arise naturally in many applications including ranking and recommendation systems, yet a fundamental question has remained open: can a single algorithm perform optimally in both stochastic and adversarial environments, without knowing which regime it faces? We answer this affirmatively, providing the first best-of-both-worlds algorithms for multi-dueling bandits under both Condorcet and Borda objectives. For the Condorcet setting, we propose \texttt{MetaDueling}, a black-box 
reduction that converts any dueling bandit algorithm into a multi-dueling 
bandit algorithm by transforming multi-way winner feedback into an unbiased 
pairwise signal. Instantiating our reduction with \texttt{Versatile-DB} yields 
the first best-of-both-worlds algorithm for multi-dueling bandits: it achieves 
$O(\sqrt{KT})$ pseudo-regret against adversarial preferences and the 
instance-optimal $O\!\left(\sum_{i \neq a^\star} \frac{\log T}{\Delta_i}\right)$ 
pseudo-regret under stochastic preferences, both simultaneously and without 
prior knowledge of the regime. For the Borda setting, we propose \AlgBorda, a stochastic-and-adversarial  algorithm that achieves $O\left(K^2 \log KT + K \log^2 T + \sum_{i: \Delta_i^{\mathrm{B}} > 0} \frac{K\log KT}{(\Delta_i^{\mathrm{B}})^2}\right)$ regret in stochastic environments and  $O\left(K \sqrt{T \log KT} + K^{1/3} T^{2/3} (\log K)^{1/3}\right)$ regret against  adversaries, again without prior knowledge of the regime.
We complement our upper bounds with matching lower bounds for the Condorcet setting. For the Borda setting, our upper bounds are near-optimal with respect to the lower bounds (within a factor of $K$) and match the best-known results in the literature.
\end{abstract}

\section{Introduction}
\label{sec:intro}

Multi-armed bandits with preference feedback have emerged as a powerful framework for applications where absolute numerical rewards are difficult to obtain, but relative comparisons are natural, such as information retrieval \citep{yue2009interactively} and recommendation systems. In the classical \emph{dueling bandits} setting, a learner selects two arms per round and observes which is preferred. Recently, \emph{multi-dueling bandits} \citep{saha2018battle,brost2016multi} have extended this framework to settings where $m \geq 2$ arms are compared simultaneously, with only the winner revealed (e.g., ranker evaluation based on user clicks).

A fundamental question in online learning is whether algorithms can achieve optimal performance in both stochastic and adversarial environments without prior knowledge of the regime. Such \emph{best-of-both-worlds} guarantees are established for multi-armed \citep{zimmert2022tsallis} and dueling bandits \citep{saha2022versatile}, but remain unexplored for multi-dueling bandits. Existing work addresses these regimes separately: \citet{saha2018battle} achieve $O(K\log T/\Delta)$ regret in stochastic multi-dueling bandits under the \emph{Condorcet} assumption, while \citet{gajane2024adversarial} achieve $O((K\log K)^{1/3}T^{2/3})$ regret in adversarial settings using \emph{Borda} scores.

\subsection{Contributions}
We provide the first best-of-both-worlds algorithms for multi-dueling bandits, addressing both Condorcet and Borda objectives:

\begin{enumerate}[leftmargin=*, itemsep=2pt]
    \item \textbf{Condorcet Setting: The \texttt{MetaDueling} Reduction.} We propose \texttt{MetaDueling} (Algorithm~\ref{alg:meta_dueling}), a black-box reduction from multi-dueling to dueling bandits. By constructing multisets containing only two distinct arms, we extract unbiased pairwise comparison information from multi-way winner feedback, as the bias from duplicate arms corresponds to an affine rescaling of the preference matrix that preserves the Condorcet winner and gap ordering (Lemma~\ref{lem:valid_rescaled}). Instantiating this with \texttt{Versatile-DB} \citep{saha2022versatile} yields the first best-of-both-worlds multi-dueling guarantee (Theorem~\ref{thm:main}): $O(\sqrt{KT})$ adversarial pseudo-regret and instance-optimal $O(\sum_{i \neq a^\star} \log T / \Delta_i)$ stochastic pseudo-regret, independently of the subset size $m$.
    
    \item \textbf{Borda Setting: The \AlgBorda \ Algorithm.} We propose \AlgBorda \ (Algorithm~\ref{alg:sao_midex}) for Borda winners. Since the Borda objective requires estimating performance against all opponents, \AlgBorda \ initially assumes stochastic preferences using successive elimination strategy to obtain unbiased estimates. It monitors for adversarial deviations via martingale concentration; upon detection, it irreversibly switches to adversarial mode. This yields $O\left(K^2 \log KT + K \log^2 T + \sum_{i: \Delta_i^{\mathrm{B}} > 0} \frac{K\log KT}{(\Delta_i^{\mathrm{B}})^2}\right)$ stochastic regret and $O\left(K \sqrt{T \log KT} + K^{1/3} T^{2/3} (\log K)^{1/3}\right)$ adversarial regret (Theorems~\ref{thm:stochastic_formal} and~\ref{thm:adversarial_formal}).
    
    \item \textbf{Lower Bounds.} We prove matching lower bounds of $\Omega(\sqrt{KT})$ for adversarial Condorcet regret and $\Omega(\sum_{i \neq a^\star} \log T / \Delta_i)$ for stochastic Condorcet regret (Theorems~\ref{thm:lower_adv} and~\ref{thm:lower_stoch}), establishing the optimality of our Condorcet guarantees. We also establish lower bounds of $\Omega(K^{1/3}T^{2/3})$ for adversarial regret and $\Omega\bigl(\sum_{i: \Delta_i^{\mathrm{B}} > 0} \frac{\log T}{(\Delta_i^{\mathrm{B}})^2}\bigr)$ for stochastic  regret under Borda objectives. 
\end{enumerate}

\subsection{Related work}
Multi-dueling bandits \citep{brost2016multi} generalize the standard dueling framework \citep{yue2009interactively} to simultaneous $m$-way comparisons. Under the challenging winner-only feedback model, \citet{saha2018battle} established gap-dependent $O(K \log T / \Delta)$ regret bounds for purely stochastic environments with Condorcet winners. More recently, \citet{gajane2024adversarial} studied the adversarial multi-dueling regime, achieving $O((K \log K)^{1/3} T^{2/3})$ regret under the Borda objective. However, neither of these approaches can automatically adapt when the nature of the environment (stochastic versus adversarial) is unknown.

The quest for algorithms that adapt seamlessly to either regime—termed \emph{best-of-both-worlds}—has seen significant breakthroughs via Tsallis entropy regularization \citep{zimmert2019optimal,zimmert2022tsallis}. \citet{saha2022versatile} successfully extended this machinery to standard dueling bandits ($m=2$), achieving optimal rates in both regimes simultaneously. Our \texttt{MetaDueling} reduction directly leverages their \texttt{Versatile-DB} algorithm to elevate these dual guarantees to the multi-dueling setting.

For the Borda setting, we instead draw upon the stochastic-and-adversarial paradigm introduced by \citet{pmlr-v49-auer16}. Rather than relying on regularization, this approach actively monitors empirical observations for adversarial deviations, permanently switching from a stochastic successive elimination strategy to a robust adversarial one upon detection. A comprehensive discussion of the broader literature, including alternative choice models and preference structures, is deferred to Appendix~\ref{app:related_work}.
\section{Preliminaries}
\label{sec:prelim}

We consider online learning over a finite set of $K$ arms $[K] := \{1,\ldots,K\}$ for $T$ rounds. At each round $t \in [T]$, the environment obliviously selects a preference matrix $P_t \in [0,1]^{K \times K}$ satisfying reciprocity ($P_t(i,j) + P_t(j,i) = 1$) and self-comparison ($P_t(i,i) = 1/2$) for all $i,j \in [K]$. The learner selects a multiset $\mathcal{A}_t \subseteq [K]$ of size $|\mathcal{A}_t| = m$ ($2 \leq m \leq K$) and observes a winning arm. In the \emph{stochastic setting}, preferences are stationary ($P_t = P$ for all $t$); in the \emph{adversarial setting}, the sequence $\{P_t\}_{t=1}^T$ is arbitrary.

\paragraph{Choice Model.} 
\label{def:pscm}
Following \citet{saha2018battle}, we assume a pairwise-subset choice model. Given a multiset $\mathcal{A}$ and preference matrix $P$, the probability that the arm at index $i \in [m]$ wins is determined by averaging pairwise preferences: $W(i \mid \mathcal{A}, P) := \sum_{j \neq i} \frac{2P(\mathcal{A}(i), \mathcal{A}(j))}{m(m-1)}$. 

\paragraph{Condorcet Setting.} 
The Condorcet framework assumes a globally dominant arm. 
\begin{assumption}[Condorcet Winner and Gap]\label{ass:condorcet}
We assume there exists a Condorcet winner $a^\star \in [K]$ such that $P_t(a^\star, i) \geq 1/2$ for all $i \in [K]$ and $t \in [T]$. We define the sub-optimality gap for arm $i$ at round $t$ as $\Delta_t(i) := P_t(a^\star, i) - 1/2$. In the stochastic setting, we write $\Delta(i) := \Delta_t(i)$ and let $\Delta_{\min} := \min_{i \neq a^\star} \Delta(i) > 0$.
\end{assumption}
The cumulative Condorcet pseudo-regret measures the average gap of selected arms: $R_T^{\mathrm{C}} := \sum_{t=1}^{T} \frac{1}{m} \sum_{j=1}^{m} \Delta_t(\mathcal{A}_t(j))$. Note that standard dueling bandit regret is simply the special case where $m=2$.

\paragraph{Borda Setting.} 
Alternatively, optimality can be measured by average pairwise performance. The Borda score of arm $i$ at round $t$ is the probability it beats an opponent selected uniformly at random: $b_t(i) := \frac{1}{K-1} \sum_{j \neq i} P_t(i,j)$.

Unlike a Condorcet winner, which is not guaranteed to exist (for instance, no Condorcet winner exists in the MSLR-WEB10k dataset \cite{Qin2010}), a Borda winner always exists. Moreover, when Condorcet and Borda winners differ, the Borda winner can better reflect the underlying preferences and is more robust to estimation errors \citep{pmlr-v38-jamieson15}. For the relationship between Condorcet setting and Borda setting, see Appendix \ref{app:prelim:relationship}.

\section{The MetaDueling Reduction}
\label{sec:algorithm}
The transition from pairwise to multi-way preferences introduces a significant structural challenge: most multi-dueling algorithms are designed for specific, narrow environments. This lack of modularity often prevents the direct application of well-optimized dueling bandit algorithms to multi-player settings, requiring new, complex strategies for every different scenarios.

We present \texttt{MetaDueling}, a black-box reduction that decouples the multi-way feedback mechanism from the core learning strategy. Our approach constructs the multiset $\mathcal{A}_t$ by populating its $m$ available slots with only two distinct arms, $i$ and $j$. By isolating these two arms, we collapse the multi-player interaction into a controlled pairwise signal, allowing multi-way winner probabilities to be mapped directly back to pairwise preferences without needing to model the internal dynamics of the selection process.

The primary advantage of \texttt{MetaDueling} is its inherent modularity. By treating the multi-dueling environment as a black-box interface, our framework can be seamlessly integrated with any standard dueling bandit learner—including stochastic, adversarial, or contextual variants. This plug-and-play architecture ensures that \texttt{MetaDueling} inherits the regret guarantees of the base learner while remaining extensible to any preference learning setting where multi-way feedback is available.

\subsection{Algorithm Description}
\label{sec:algorithm:desc}

The reduction operates as follows: at each round, the base dueling bandit 
algorithm proposes two arms $(x_t, y_t)$. We construct a multiset containing 
only these two arms, with counts $(n_x, n_y)$ summing to $m$. When $m$ is even, 
we set $n_x = n_y = m/2$. When $m$ is odd, we randomize: with probability $1/2$, 
we set $(n_x, n_y) = (\lceil m/2 \rceil, \lfloor m/2 \rfloor)$, and with 
probability $1/2$, we swap. This symmetrization ensures unbiased feedback 
in expectation.

\begin{algorithm}[t]
\caption{\texttt{MetaDueling}: Reduction from Multi-Dueling to Dueling Bandits}
\label{alg:meta_dueling}
\begin{algorithmic}[1]
\Require Base dueling bandit algorithm $\mathcal{D}$, horizon $T$, arms $K$, subset size $m$
\State Initialize base algorithm $\mathcal{D}$ (e.g., \texttt{Versatile-DB})
\For{$t = 1, 2, \ldots, T$}
    \State \textbf{// Query base algorithm for arm pair}
    \State Let $(q_t^{(1)}, q_t^{(2)})$ be the distributions maintained by $\mathcal{D}$
    \State Sample $x_t \sim q_t^{(1)}$ and $y_t \sim q_t^{(2)}$ independently
    \State \textbf{// Construct multiset with symmetrization}
    \State Sample $B_t \sim \mathrm{Bernoulli}(1/2)$
    \If{$m$ is even}
        \State $(n_x, n_y) \gets (m/2, m/2)$
    \ElsIf{$B_t = 1$}
        \State $(n_x, n_y) \gets (\lceil m/2 \rceil, \lfloor m/2 \rfloor)$
    \Else
        \State $(n_x, n_y) \gets (\lfloor m/2 \rfloor, \lceil m/2 \rceil)$
    \EndIf
    \State Construct $\mathcal{A}_t = \{\underbrace{x_t,\ldots,x_t}_{n_x},\, \underbrace{y_t,\ldots,y_t}_{n_y}\}$
    \State \textbf{// Play and observe feedback}
    \State Play $\mathcal{A}_t$ and observe winner index $I_t \in [m]$
    \State Set $o_t \gets \ind [\mathcal{A}_t(I_t) = x_t]$
    \State \textbf{// Feed outcome to base algorithm}
    \State Update $\mathcal{D}$ with outcome $o_t$ for pair $(x_t, y_t)$
\EndFor
\end{algorithmic}
\end{algorithm}

\subsection{Rescaled Preferences}
\label{sec:algorithm:rescaled}

The multiset construction introduces bias from ties between identical arms: 
when multiple copies of $x_t$ compete, some probability mass goes to 
$x_t$-vs-$x_t$ comparisons (which $x_t$ wins with probability $1/2$) rather 
than $x_t$-vs-$y_t$ comparisons. We show that this bias corresponds exactly 
to an affine rescaling of the preference matrix.

\begin{definition}[Rescaling Constants]
\label{def:rescaling}
Define the rescaling constants:
\begin{equation}
  \beta_m :=
  \begin{cases}
    \dfrac{m}{2(m-1)} & \text{if } m \text{ is even},\\[8pt]
    \dfrac{m+1}{2m} & \text{if } m \text{ is odd},
  \end{cases}
  \qquad
  \alpha_m := \frac{1 - \beta_m}{2}.
\end{equation}
\end{definition}

\begin{definition}[Rescaled Preference Matrix]
\label{def:rescaled_pref}
The rescaled preference matrix is:
\begin{equation}
\label{eq:rescaled_pref}
  \widehat{P}_t(i, j) := \alpha_m + \beta_m \, P_t(i, j).
\end{equation}
\end{definition}
The rescaled preference matrix $\widehat{P}_t$ is a valid preference matrix and preserves the Condorcet winner (see Appendix \ref{app:algorithm:validity} for the proof).

The following lemma shows that the observed outcome $o_t$ is an unbiased 
estimate of the rescaled preference, not the original preference.

\begin{lemma}[Unbiased Feedback]
\label{lem:feedback}
Under Algorithm~\ref{alg:meta_dueling}, the observed outcome satisfies:
\begin{equation*}
    \mathbb{E}[o_t \mid x_t, y_t] = \widehat{P}_t(x_t, y_t) = \alpha_m + \beta_m P_t(x_t, y_t).
\end{equation*}
\end{lemma}

\begin{proof}
See Appendix~\ref{app:feedback}.
\end{proof}

\begin{remark}[Implicit Rescaling]
\label{rem:rescaling}
The raw outcome $o_t \in \{0,1\}$ is \emph{not} an unbiased estimate of 
$P_t(x_t, y_t)$ due to the tie bias. However, $o_t$ \emph{is} an unbiased 
estimate of $\widehat{P}_t(x_t, y_t) \in [0,1]$. Since the rescaled preferences 
preserve all structural properties needed by the base algorithm 
(Lemma~\ref{lem:valid_rescaled}), we can feed $o_t$ directly to $\mathcal{D}$ 
without explicit transformation.
\end{remark}

\subsection{Gap Rescaling}
\label{sec:algorithm:gap}

The sub-optimality gaps are also rescaled by the factor $\beta_m$.

\begin{lemma}[Gap Rescaling]
\label{lem:gap_rescaling}
For any arm $i \in [K]$, the rescaled gap satisfies:
\begin{equation*}
  \widehat\Delta_t(i) := \widehat{P}_t(a^\star, i) - \tfrac{1}{2}
  = \beta_m \, \Delta_t(i).
\end{equation*}
\end{lemma}

\begin{lemma}[Bound on Rescaling Factor]
\label{lem:beta_bound}
For all $m \geq 2$, the rescaling factor satisfies $\beta_m \geq \frac{1}{2}$, 
with equality as $m \to \infty$.
\end{lemma}

For proofs of Lemma \ref{lem:gap_rescaling} and Lemma \ref{lem:beta_bound}, see Appendix \ref{app:ProofsGapRescaling}.
\subsection{Regret Equivalence}
\label{sec:algorithm:regret}

The key result of this section establishes that the multi-dueling regret equals 
the dueling regret in rescaled preferences, up to the factor $1/\beta_m \leq 2$.

\begin{lemma}[Regret Equivalence]
\label{lem:regret_equiv}
For any sequence of preference matrices $P_1,\ldots,P_T$ satisfying 
Assumption~\ref{ass:condorcet}:
\begin{equation}
    \mathbb{E}[R_T^{\mathrm{C}}] = \frac{1}{\beta_m}\,\mathbb{E}[\widehat{R}_T^{\,\mathrm{DB}}],
\end{equation}
where $\widehat{R}_T^{\,\mathrm{DB}} := \sum_{t=1}^T \frac{1}{2}(\widehat\Delta_t(x_t) + \widehat\Delta_t(y_t))$
is the pseudo-regret of the base algorithm under $\widehat{P}_t$.
\end{lemma}
For the proof of Lemma \ref{lem:regret_equiv}, see Appendix \ref{app:RegEquivRescaledPref}.

\begin{remark}[Black-Box Nature of the Reduction]
\label{rem:blackbox} We only require that the base algorithm 
$\mathcal{D}$ accepts binary outcomes $o_t \in \{0,1\}$ and achieves bounded 
pseudo-regret on valid preference matrices. No modification to $\mathcal{D}$'s 
internals is needed. This modularity means that future improvements to dueling 
bandit algorithms automatically transfer to multi-dueling bandits via our reduction.
\end{remark}

\begin{remark}[Special Case: $m = 2$]
\label{rem:m_equals_2}
For $m = 2$, we have $\alpha_2 = 0$ and $\beta_2 = 1$, so $\widehat{P}_t = P_t$. 
The multiset $\mathcal{A}_t = \{x_t, y_t\}$ is a standard pair, and 
Algorithm~\ref{alg:meta_dueling} reduces exactly to the base algorithm 
$\mathcal{D}$.
\end{remark}

\section{Main Results: Condorcet Setting}
\label{sec:condorcet_results}
While previous literature has addressed stochastic and adversarial multi-dueling environments in isolation, we demonstrate that feeding our rescaled pairwise signals into a best-of-both-worlds base algorithm seamlessly bridges this divide. This approach eliminates the need to manually adapt to environmental shifts, as Alg. \ref{alg:meta_dueling} inherits the robust properties of the underlying learner.

By instantiating \texttt{MetaDueling} with \texttt{Versatile-DB} \citep{saha2022versatile}, we obtain the first unified guarantees for multi-dueling bandits: optimal adversarial pseudo-regret and instance-optimal stochastic pseudo-regret, achieved simultaneously and without prior knowledge of the regime. This reduction provides a rigorous theoretical foundation for multi-way feedback that remains valid across settings.

The strength of Theorem~\ref{thm:main} lies in the rigorous mapping between multi-way winner probabilities and the rescaled pairwise preference matrix $\widehat{P}_t$. By proving that the symmetrized multiset construction in \texttt{MetaDueling} yields an unbiased pairwise signal, we ensure that the base learner $\mathcal{D}$ operates on a valid dueling bandit instance.  The core of the proof establishes a strict linear regret equivalence $1/\beta_m$, which allows the $O(\sqrt{KT})$ adversarial and $O(\sum \log T / \Delta_i)$ stochastic guarantees of \texttt{Versatile-DB} to be transferred exactly to the multi-dueling setting. This transformation demonstrates that the complexity of the $m$-way interaction is captured by a universal constant factor ($1/\beta_m \leq 2$), preserving the optimality of the underlying dueling bandit algorithm regardless of the number of arms compared simultaneously.

\subsection{Best-of-Both-Worlds Guarantee}
\label{sec:condorcet:main}

\begin{theorem}[Main Result: Best-of-Both-Worlds for Condorcet Multi-Dueling]
\label{thm:main}
Under Assumption~\ref{ass:condorcet}, \texttt{MetaDueling} 
(Algorithm~\ref{alg:meta_dueling}) instantiated with \texttt{Versatile-DB} 
satisfies the following simultaneously for any $T \geq 1$:

\textbf{(i) Adversarial pseudo-regret bound.} For any oblivious sequence of 
preference matrices $P_1,\ldots,P_T$:
\begin{equation}
\label{eq:adv_bound}
  \mathbb{E}[R_T^{\mathrm{C}}]
  \leq \frac{1}{\beta_m}\left(4\sqrt{KT} + 1\right)
  \leq 8\sqrt{KT} + 2.
\end{equation}

\textbf{(ii) Bound for self-bounding preference matrices.} If the preference matrices 
satisfy the self-bounding condition
\begin{equation}
\label{eq:self_bound}
  \mathbb{E}[\widehat{R}_T^{\,\mathrm{DB}}]
  \geq
  \frac{1}{2}\,\mathbb{E}\!\left[\sum_{t=1}^{T}\sum_{i \neq a^\star}
  \bigl(p_{+1,t}(i) + p_{-1,t}(i)\bigr)\,\widehat\Delta(i)\right],
\end{equation}
where $p_{+1,t}$ and $p_{-1,t}$ are the arm distributions of the two players 
in \texttt{Versatile-DB}, then:
\begin{equation}
\label{eq:stoch_bound}
  \mathbb{E}[R_T^{\mathrm{C}}] \leq \frac{1}{\beta_m^2}
  \sum_{i \neq a^\star} \frac{16\log T + 48}{\Delta_i}
  + \frac{4\log T + 12}{\beta_m} + \frac{8\sqrt{K}}{\beta_m}.
\end{equation}

Both bounds hold simultaneously without prior knowledge of the regime.
\end{theorem}
The proof of Theorem~\ref{thm:main} leverages the insight that the algorithm’s observed outcomes can be rescaled to form valid dueling bandit feedback, allowing the application of existing guarantees for the \texttt{Versatile-DB} algorithm. In the adversarial case, this directly yields sublinear pseudo-regret bounds, while in the self-bounding setting, the rescaling of gaps enables the translation of known gap-dependent guarantees to the Condorcet setting. This approach simultaneously harnesses prior work on dueling bandits and our novel rescaling technique, producing the first best-of-both-worlds bounds for multi-dueling bandits under Condorcet objectives.

For the complete proof of Theorem \ref{thm:main}, see Appendix \ref{app:ProofMainCondorcetThm}.

\subsection{Stochastic Condorcet Setting}
\label{sec:condorcet:stochastic}

In the stochastic setting with a fixed Condorcet winner, the self-bounding 
condition~\eqref{eq:self_bound} is automatically satisfied, yielding 
instance-optimal regret.

\begin{corollary}[Instance-Optimal Stochastic Regret]
\label{cor:stochastic}
In the stochastic setting with a fixed preference matrix $P_t = P$ for all 
$t \in [T]$ and Condorcet winner $a^\star$, the self-bounding 
condition~\eqref{eq:self_bound} holds with equality. Consequently:
\begin{equation}
\label{eq:stoch_instance_opt}
  \mathbb{E}[R_T^{\mathrm{C}}] \leq
  \sum_{i\neq a^\star} \frac{64\log T + 192}{\Delta_i}
  + 8\log T + 24 + 16\sqrt{K}
  = O\!\left(\sum_{i \neq a^\star} \frac{\log T}{\Delta_i}\right).
\end{equation}
\end{corollary}

\begin{proof}
In the stochastic setting, $\Delta_t(i) = \Delta(i)$ for all $t \in [T]$, and 
the Condorcet winner $a^\star$ is fixed. The self-bounding condition holds 
because pulling suboptimal arms directly contributes to regret proportional 
to their gaps.

The bound follows from Theorem~\ref{thm:main}(ii) with $1/\beta_m \leq 2$ and 
$1/\beta_m^2 \leq 4$.
\end{proof}

\begin{remark}[Comparison with Prior Work]
\label{rem:comparison}
\citet{saha2018battle} achieved $O(K\log T / \Delta_{\min})$ regret for 
stochastic multi-dueling bandits with Condorcet winners. Our bound 
$O(\sum_{i \neq a^\star} \log T / \Delta_i)$ is instance-optimal: it depends 
on individual gaps rather than just the minimum gap, providing tighter 
guarantees when gaps are heterogeneous. Moreover, our algorithm simultaneously 
achieves $O(\sqrt{KT})$ adversarial regret, which \citet{saha2018battle} do not address.
\end{remark}

\subsection{Lower Bounds}
\label{sec:condorcet:lower}

We establish matching lower bounds that demonstrate the optimality of our 
guarantees.

\begin{theorem}[Adversarial Lower Bound]
\label{thm:lower_adv}
For any multi-dueling bandit algorithm and any $T \geq K \geq 4$, there exists 
an adversarial instance satisfying Assumption~\ref{ass:condorcet} such that:
\begin{equation}
    \mathbb{E}[R_T^{\mathrm{C}}] \geq \Omega(\sqrt{KT}).
\end{equation}
\end{theorem}
\begin{proof}[Proof Sketch]
The proof of the adversarial lower bound proceeds via a reduction from dueling bandits. First, any multi-dueling bandit algorithm is converted into a standard dueling bandit algorithm by randomly selecting a pair of arms from the chosen subset and using the observed duel outcome to simulate multi-dueling feedback. This ensures that the expected regret of the constructed dueling algorithm matches that of the original multi-dueling algorithm. Next, the standard minimax lower bound for adversarial multi-armed bandits is applied, and via known reductions, this yields a hard adversarial instance with a best arm in hindsight. Finally, the reduction shows that if a multi-dueling algorithm could achieve lower regret, it would violate the dueling bandit lower bound, establishing the desired $\Omega(\sqrt{KT})$ bound for multi-dueling setting. 
\end{proof}

\begin{theorem}[Stochastic Lower Bound]
\label{thm:lower_stoch}
For any multi-dueling bandit algorithm, there exists a stochastic instance 
with Condorcet winner $a^\star$ such that:
\begin{equation}
    \mathbb{E}[R_T^{\mathrm{C}}] \geq \Omega\!\left(\sum_{i \neq a^\star} 
    \frac{\log T}{\Delta_i}\right).
\end{equation}
\end{theorem}

\begin{proof}[Proof Sketch]
The proof follows the same reduction structure as Theorem~\ref{thm:lower_adv}, 
but applies the stochastic dueling bandit lower bound of \citet{komiyama2015regret}:
\[
    \liminf_{T \to \infty} \frac{\mathbb{E}[R_T^{\mathrm{DB}}]}{\log T}
    \geq \sum_{i \neq a^\star} \frac{1}{2\,\mathrm{KL}(P(a^\star,i) \| \tfrac{1}{2})}.
\]

For gaps $\Delta_i = P(a^\star, i) - \frac{1}{2}$, a Taylor expansion gives 
$\mathrm{KL}(P(a^\star,i) \| \frac{1}{2}) = \Theta(\Delta_i^2)$, yielding 
the $\Omega(\sum_{i \neq a^\star} \log T / \Delta_i)$ lower bound.

See Appendix~\ref{app:lower} for the complete proof.
\end{proof}

\begin{remark}[Optimality]
\label{rem:optimality}
The matching upper and lower bounds establish that our algorithm achieves 
optimal pseudo-regret in both adversarial and stochastic settings. The 
adversarial bound $\Theta(\sqrt{KT})$ matches the minimax rate for multi-armed 
bandits, and the stochastic bound $\Theta(\sum_{i \neq a^\star} \log T / \Delta_i)$ 
matches the instance-optimal rate for dueling bandits.
\end{remark}

\begin{remark}[Pseudo-Regret vs.\ High-Probability Bounds]
\label{rem:pseudo_regret}
Our bounds are in expectation (pseudo-regret). As shown in \citet{pmlr-v49-auer16}, no algorithm can simultaneously achieve 
optimal high-probability bounds in both stochastic and adversarial regimes. 
Our pseudo-regret formulation is therefore the natural target for 
best-of-both-worlds guarantees.
\end{remark}

\section{Borda Setting: The \AlgBorda \ Algorithm}
\label{sec:borda}

We now turn to multi-dueling bandits with Borda objectives. Unlike the 
Condorcet setting, the Borda objective requires estimating each arm's average 
performance against \emph{all} opponents, not just the specific arm it was 
compared against. 

We propose \AlgBorda \ (stochastic-and-Adversarial Multi-Dueling 
with EXploration), an algorithm based on the stochastic-and-adversarial 
paradigm of \citet{pmlr-v49-auer16}. The algorithm initially assumes a 
stochastic environment, monitors for adversarial deviations, and switches 
to a robust adversarial strategy upon detection.
 
The Borda setting introduces unique challenges: estimating each arm’s Borda score requires comparisons against all other arms, no Condorcet winner is guaranteed (so preferences may be cyclic or change over time), and multi-dueling feedback provides limited information for importance-weighted estimation. To address these, the approach combines decaying uniform opponent sampling with probability  $p_t = \min(1, \frac{K \log t}{t})$ to maintain unbiased Borda estimates, a two-phase exploration-exploitation scheme, martingale-based deviation detection to identify adversarial shifts, and an EXP3 fallback to ensure robust adversarial regret.

\subsection{Algorithm Description}
\label{sec:borda:algorithm}

The \AlgBorda\ algorithm operates in two modes with an irreversible mode switch, leveraging the \texttt{MetaDueling} reduction (Algorithm~\ref{alg:meta_dueling}) as a black box to convert multi-dueling feedback into pairwise comparisons.

\paragraph{Feedback Transformation.}
 The \texttt{MetaDueling} reduction returns a raw outcome $o_t \in \{0,1\}$ indicating whether arm $x_t$ won. To obtain an unbiased estimate of the true preference probability $P(x_t, y_t)$, we apply the affine transformation:
\begin{equation}
    g_t = \frac{o_t - \alpha_m}{\beta_m},
\end{equation}
where $\alpha_m$ and $\beta_m$ are the bias and scaling constants from Lemma~\ref{lem:feedback} that depend on the subset size $m$. This yields $\mathbb{E}[g_t \mid x_t, y_t] = P(x_t, y_t)$, enabling standard bandit estimation techniques despite the non-standard multi-dueling feedback model.

\paragraph{Stochastic Mode.}
The algorithm begins in stochastic mode, which operates via three mechanisms:

\textit{Stage 0: Baseline Exploration ($t \leq T_0$).}
For the first $T_0 = 4K(K-1)\lceil\log(2K^2T^2/\delta)\rceil$ rounds, the algorithm explores all ordered pairs via a round-robin schedule. This pure exploration phase is necessary to establish a high-confidence empirical baseline $\hat{P}_{T_0}(i,j)$ for the deviation detection mechanism.

\textit{Stage 1: Successive Elimination ($t > T_0$).}
The algorithm maintains an active set of candidate arms $\mathcal{C}$, initialized to $[K]$. In each round $t$, to guarantee unbiased Borda estimates, the algorithm selects $x_t \in \mathcal{C}$ (e.g., via round-robin over $\mathcal{C}$) and samples the opponent $y_t \sim \text{Uniform}([K] \setminus \{x_t\})$. Suboptimal arms are permanently eliminated from $\mathcal{C}$ once sufficient evidence accumulates. Specifically, arm $j \in \mathcal{C}$ is eliminated if there exists another arm $i \in \mathcal{C}$ such that $\hat{b}_t(i) - \text{conf}_t(i) > \hat{b}_t(j) + \text{conf}_t(j)$, where $\text{conf}_t(i) = \sqrt{\frac{20K\log(2K^2 T/\delta)}{N_t(i)}}$. Once $|\mathcal{C}| = 1$, letting $\mathcal{C} = \{i^\star\}$, the algorithm exploits by playing $x_t = y_t = i^\star$.

\textit{Stage 2: Background Monitoring ($t > T_0$).}
To ensure the martingale deviation detection functions correctly in case of adversarial shifts, the algorithm forces random exploration with probability $p_t = \min(1, \frac{K \log t}{t})$ by sampling $x_t, y_t \sim \text{Uniform}([K])$. These uniform samples feed directly into the deviation detection statistics.

\paragraph{Deviation Detection.}
Throughout exploitation, the algorithm monitors each pair $(i,j)$ for deviation from the learned model by tracking:
\begin{equation}
    D_t(i,j) := \left|S_t^{\text{obs}}(i,j) - S_{T_0}^{\text{obs}}(i,j) 
    - (N_t(i,j) - N_{T_0}(i,j)) \cdot \hat{P}_{T_0}(i,j)\right|.
\end{equation}
A deviation triggers when $D_t(i,j) > \tau(\sqrt{N_t(i,j) - N_{T_0}(i,j)} + 1)$ for threshold $\tau = \sqrt{8\log(K^2 T^2/\delta)}$.

\paragraph{Adversarial Mode.}
Upon detecting deviation, the algorithm irreversibly switches to EXP3 with explicit uniform exploration following \cite{gajane2024adversarial}. The sampling distribution mixes a Gibbs distribution over cumulative importance-weighted scores with uniform exploration:
\begin{equation}
    q_t(i) = (1-\gamma) \frac{\exp(\eta \tilde{S}_t(i))}{\sum_{j=1}^K \exp(\eta \tilde{S}_t(j))} + \frac{\gamma}{K},
\label{eq:exp3_mixture}    
\end{equation}
ensuring $q_t(i) \geq \gamma/K$ for bounded inverse probability weights. The importance-weighted Borda score update uses the transformed feedback $g_t$ and accounts for the sampling probabilities of both arms in the pair.

\begin{algorithm}[t]
\caption{\AlgBorda: Stochastic-and-Adversarial Multi-Dueling Bandits}
\label{alg:sao_midex}
\begin{algorithmic}[1]
\Require Arms $K$, horizon $T$, subset size $m$, confidence $\delta = 1/T^2$
\State \textbf{Initialize:} $\text{mode} \gets \textsc{Stoch}$, $\mathcal{C} \gets [K]$, $T_0 \gets 4K(K-1) \lceil \log(2K^2 T^2/\delta) \rceil$
\State \textbf{Init Stats:} $N, S^{\text{obs}}, \tilde{S} \gets 0$, $\hat{P} \gets 1/2$
\For{$t = 1, \ldots, T$}
    \If{$\text{mode} = \textsc{Stoch}$ \textbf{and} $t \leq T_0$} \Comment{Stage 0: Baseline Exploration}
        \State $(x_t, y_t) \gets$ \Call{RoundRobin}{$t$} 
    \ElsIf{$\text{mode} = \textsc{Stoch}$} \Comment{Stages 1 \& 2: Elimination \& Monitoring}
        \State Sample $u_t \sim \text{Uniform}(0,1)$
        \If{$u_t < \min(1, \frac{K \log t}{t})$}
            \State $(x_t, y_t) \sim \text{Uniform}([K] \times [K])$
        \ElsIf{$|\mathcal{C}| > 1$}
            \State Pick $x_t \in \mathcal{C}$ (round-robin), $y_t \sim \text{Uniform}([K] \setminus \{x_t\})$
        \Else
            \State $(x_t, y_t) \gets (i^\star, i^\star)$ where $\mathcal{C} = \{i^\star\}$
        \EndIf
    \Else
        \State $(x_t, y_t) \gets$ \Call{EXP3Select}{} \Comment{Sample from $q$ in Eq.~(\ref{eq:exp3_mixture})}
    \EndIf
    \State $o_t \gets \texttt{MetaDueling}(x_t, y_t, m)$; $g_t \gets (o_t - \alpha_m)/\beta_m$; Update $N, S^{\text{obs}}, \hat{b}$
    \State \textbf{if} $t = T_0$ \textbf{then} $\hat{P}_{T_0}(i,j) \gets S^{\text{obs}}(i,j) / N(i,j)$ for all $(i,j)$ \Comment{Lock baseline}
    \If{$\text{mode} = \textsc{Stoch}$ \textbf{and} $t > T_0$}
        \State Eliminate $j$ from $\mathcal{C}$ if $\exists\, i \in \mathcal{C}$ s.t. $\hat{b}_t(i) - \text{conf}_t(i) > \hat{b}_t(j) + \text{conf}_t(j)$
        \State \textbf{if} $\exists\, (i,j)\!:\! D_t(i,j) \!>\! \tau(\sqrt{N_t(i,j) \!-\! N_{T_0}(i,j)} \!+\! 1)$ \textbf{then} $\text{mode} \!\gets\! \textsc{Adv}$; $\tilde{S} \!\gets\! 0$
    \ElsIf{$\text{mode} = \textsc{Adv}$}
        \State $\tilde{S}(x_t) \gets \frac{\ind(i = x_t)}{K q_t(i)} \sum_{j \in [K]} \frac{\ind(j = y_t) g_t}{q_t(j)}$
    \EndIf
\EndFor
\end{algorithmic}
\end{algorithm}

\subsection{Importance-Weighted Borda Estimation}
\label{sec:borda:estimation}

In adversarial mode, we use importance-weighted estimators for the shifted 
Borda score.

\begin{definition}[Importance-Weighted Borda Estimator]
\label{def:iw_borda}
Given arms $x_t, y_t \sim q_t$ and transformed feedback $g_t$, the 
importance-weighted shifted Borda score estimate is:
\begin{equation}
\hat{s}_t(i) := \frac{\ind(i = x_t)}{K q_t(i)} \sum_{j \in [K]} \frac{\ind(j = y_t) g(m, o_t, x_t)}{q_t(j)}
\end{equation}
\end{definition}

\begin{lemma}[Unbiased Estimation]
\label{lem:iw_unbiased}
The estimator $\hat{s}_t(i)$ is unbiased:
\begin{equation}
    \mathbb{E}[\hat{s}_t(i) \mid q_t] = s_t(i) = \frac{1}{K}\sum_{j \in [K]} P_t(i,j).
\end{equation}
\end{lemma}

\begin{proof}
Taking expectations over $x_t, y_t \stackrel{\text{i.i.d.}}{\sim} q_t$:
\begin{align*}
\mathbb{E}[\hat{s}_t(i) \mid q_t] 
&= \sum_{x,y} q_t(x) q_t(y) \cdot \frac{\ind(i=x) \cdot \mathbb{E}[g_t \mid x,y]}{K \cdot q_t(i) \cdot q_t(y)} \\
&= \sum_y q_t(y) \cdot \frac{P_t(i,y)}{K \cdot q_t(y)} 
= \frac{1}{K} \sum_y P_t(i,y) = s_t(i),
\end{align*}
where we used $\mathbb{E}[g_t \mid x_t, y_t] = P_t(x_t, y_t)$ from the 
feedback transformation. 
\end{proof}

\subsection{Main Results: Borda Setting}
\label{sec:borda:results}

We now state our main results for the Borda setting. The proofs rely on 
concentration arguments for the stochastic phase and EXP3 
analysis for the adversarial phase.

\begin{theorem}[Stochastic Regret Bound]
\label{thm:stochastic_formal}
In the stochastic setting ($P_t = P$ for all $t$), Algorithm~\ref{alg:sao_midex} achieves:
\begin{equation}
\mathbb{E}[R_T^{\mathrm{B}}] \leq O\left(K^2 \log KT + K \log^2 T + \sum_{i: \Delta_i^{\mathrm{B}} > 0} \frac{K\log KT}{(\Delta_i^{\mathrm{B}})^2}\right)
\end{equation}
\end{theorem}

\begin{proof}[Proof Sketch]
Conditioned on the high-probability event $\mathcal{E}$ where confidence bounds hold and no false adversarial switch triggers, we decompose the regret into four phases. The baseline exploration of $T_0$ rounds contributes $O(K^2 \log KT)$. During successive elimination, selecting a suboptimal arm $i$ against a uniformly sampled opponent incurs $\Theta(1)$ instantaneous regret; by Lemma~\ref{lem:ucb_valid}, arm $i$ is eliminated after at most $O(K \log KT / (\Delta_i^{\mathrm{B}})^2)$ pulls, contributing $O\bigl(\sum_{i \neq i^\star} \frac{K \log KT}{(\Delta_i^{\mathrm{B}})^2}\bigr)$. Once only the optimal arm remains, exploitation yields exactly $0$ regret. Finally, background monitoring forces uniform sampling with probability $p_t = \min(1, \frac{K \log t}{t})$, adding at most $\sum_{t=1}^T p_t \leq O(K \log^2 T)$. Adding these components yields the bound. See Appendix~\ref{app:borda:main} for full details.
\end{proof}

\begin{theorem}[Adversarial Regret Bound]
\label{thm:adversarial_formal}
Under any oblivious adversarial sequence $P_1, \ldots, P_T$, 
Algorithm~\ref{alg:sao_midex} achieves:
\begin{equation}
    \mathbb{E}[R_T^{\mathrm{B}}] \leq O\left(K \sqrt{T \log KT} + K^{1/3} T^{2/3} (\log K)^{1/3}\right)
\end{equation}
\end{theorem}

\begin{proof}[Proof Sketch]
We analyze two cases based on whether a mode switch occurs. If a switch occurs at round $t_{\text{sw}} \leq T$, the regret comprises the $O(K^2 \log KT)$ exploration overhead, the undetected pre-switch manipulation strictly bounded by the martingale detection threshold $\mathcal{R}_{\text{pre}} \leq O(K\sqrt{T \log KT})$, and the post-switch EXP3 exploration which contributes $\mathbb{E}[\mathcal{R}_{\text{post}}] \leq O(K^{1/3}T^{2/3}(\log K)^{1/3})$ (Lemma~\ref{lem:exp3_regret}). Alternatively, if no switch occurs, the adversary's cumulative deviation is continuously bounded by the threshold, limiting regret to $O(K\sqrt{T \log KT})$. Full details are in Appendix~\ref{app:borda:main}.
\end{proof}

\begin{corollary}[Best-of-Both-Worlds for Borda Setting]
\label{cor:borda_bobw}
Without prior knowledge of the environment type, Algorithm~\ref{alg:sao_midex} simultaneously achieves:
\begin{enumerate}[label=(\roman*), leftmargin=*, itemsep=2pt]
    \item \textbf{Stochastic:}$\mathbb{E}[R_T^{\mathrm{B}}] \leq O\left(K^2 \log KT + K \log^2 T + \sum_{i: \Delta_i^{\mathrm{B}} > 0} \frac{K\log KT}{(\Delta_i^{\mathrm{B}})^2}\right)$
    \item \textbf{Adversarial:} $\mathbb{E}[R_T^{\mathrm{B}}] \leq O\bigl(K \sqrt{T \log KT} + K^{1/3}T^{2/3}(\log K)^{1/3}\bigr)$
\end{enumerate}
\end{corollary}

\subsection{Lower Bounds}
\begin{remark}[Optimality in the Borda Setting]
\label{rem:borda_lower_bounds} 

\begin{enumerate}
    \item \textbf{Stochastic Setting:} Instance-dependent lower bounds imply an $\Omega\bigl(\sum_{i: \Delta_i^{\mathrm{B}} > 0} \frac{\log T}{(\Delta_i^{\mathrm{B}})^2}\bigr)$ limit. Our \AlgBorda\ algorithm achieves this gap-dependent rate asymptotically, though it incurs an initial $O(K^2 \log KT)$ exploration overhead necessary to unbiasedly estimate global Borda scores.
    \item \textbf{Adversarial Setting:} As established by \citet{gajane2024adversarial}, the expected regret for any multi-dueling bandit algorithm measured against an adversarial Borda winner is lower bounded by $\Omega(K^{1/3}T^{2/3})$. Our  adversarial guarantee of $O(K^{1/3}T^{2/3}(\log K)^{1/3})$ therefore matches the theoretical limit of this setting up to logarithmic factors. 
\end{enumerate}
\end{remark}
The proofs for these fundamental limits follow a similar structure to our Condorcet lower bounds (Theorems~\ref{thm:lower_adv} and~\ref{thm:lower_stoch}), utilizing a reduction argument that simulates multi-dueling feedback to invoke dueling bandit lower bounds.
\subsection{Key Technical Lemmas}
\label{sec:borda:lemmas}

We state the key lemmas underlying our analysis. Full proofs are deferred to 
Appendix~\ref{app:borda}.

\begin{lemma}[Exploration Guarantees]
\label{lem:exploration_guarantee}
With $T_0 = 4K(K-1)\lceil\log(2K^2T^2/\delta)\rceil$, after exploration, 
with probability at least $1 - \delta/2$:
\begin{enumerate}[label=(\roman*), leftmargin=*, itemsep=1pt]
    \item Each pair $(i,j)$ is observed at least $n_0 = \Theta(\log T)$ times.
    \item $|\hat{P}_{T_0}(i,j) - P(i,j)| \leq O(\sqrt{\log T / n_0})$ for all pairs.
    \item $|\hat{b}_{T_0}(i) - b(i)| \leq O(\sqrt{\log T / n_0})$ for all arms.
\end{enumerate}
\end{lemma}

\begin{lemma}[No False Alarm in Stochastic Setting]
\label{lem:no_false_alarm}
In the stochastic setting, with probability at least $1 - \delta/2$, the 
deviation detection never triggers throughout $T$ rounds.
\end{lemma}

\begin{lemma}[Adversarial Detection Guarantee]
\label{lem:detection_guarantee}
If the adversary's cumulative preference shift on pair $(i,j)$ exceeds 
$3\gamma(\sqrt{n} + 1)$ where $n$ is the post-exploration comparison count, 
then deviation is detected with probability at least $1 - \delta/(K^2 T^2)$.
\end{lemma}

\begin{lemma}[Confidence Validity]
\label{lem:ucb_valid}
With probability at least $1 - \delta/4$, for all arms $i$ and rounds $t$:
\begin{equation}
    |\hat{b}_t(i) - b(i)| \leq \sqrt{\frac{2K\log(2K^2T/\delta)}{N_t(i)}}.
\end{equation}
\end{lemma}

\begin{lemma}[Elimination Sample Complexity]
\label{lem:suboptimal_pulls}
Conditioned on confidence validity, each suboptimal arm $j$ is eliminated from $\mathcal{C}$ after at most:
\begin{equation}
    N_T(j) \leq \frac{32K\log(2K^2T/\delta)}{(\Delta_j^{\mathrm{B}})^2} + 1 \text{ pulls.}
\end{equation}
\end{lemma}

\begin{lemma}[EXP3 Regret in Adversarial Mode]
\label{lem:exp3_regret}
Starting from any round $t_0$ with fresh initialization, the EXP3 algorithm with explicit exploration parameters $\eta = \Theta(T^{-2/3})$ and $\gamma = \Theta(T^{-1/3})$ achieves:
\begin{equation}
    \mathbb{E}\left[\sum_{t=t_0}^{T} \left(s_t(i^\star) - 
    \frac{s_t(x_t) + s_t(y_t)}{2}\right)\right] 
    \leq O\left(K^{1/3}(T - t_0 + 1)^{2/3} (\log K)^{1/3}\right).
\end{equation}
\end{lemma}
\section{Extensions}
\label{sec:extensions}

Our algorithmic framework flexibly accommodates several practical extensions. 

\textbf{Time-Varying Subset Sizes:} Theorems~\ref{thm:main},~\ref{thm:stochastic_formal}, and~\ref{thm:adversarial_formal} extend immediately to the setting where the subset size $m_t \in \{2,\ldots,K\}$ varies arbitrarily across rounds. Because the rescaling factor $\beta_{m_t}$ depends only on the current $m_t$ and the feedback transformations remain unbiased pointwise, the regret bounds carry through with identical constants. Formal statements and proofs are deferred to Appendix~\ref{app:extensions}.

\textbf{Contextual, Batched, and High-Probability Settings:} The black-box nature of the \texttt{MetaDueling} reduction allows it to effortlessly inherit capabilities for contexts, batched feedback, and high-probability bounds simply by swapping the base algorithm (e.g., \texttt{Versatile-DB}) for a variant designed for those settings. For the Borda setting, extending \AlgBorda\ to contextual or high-probability domains requires modifying the global Borda score estimators and concentration thresholds, which presents an interesting direction for future work.
\section{Discussion and Conclusion}
\label{sec:discussion}

We have presented the first best-of-both-worlds algorithms for multi-dueling 
bandits, addressing both Condorcet and Borda objectives. 



The $O(K^2 \log KT)$ exploration phase in \AlgBorda \ is required to 
learn all pairwise preferences. This overhead dominates for small $T$ or 
large $K$. Structured preference models (e.g., low-rank, parametric) could 
potentially reduce this overhead.
Both algorithms construct multisets containing only two distinct arms, 
potentially wasting the ability to compare $m > 2$ arms simultaneously. 
A more sophisticated reduction might extract more information from 
multi-way comparisons, though this would likely sacrifice the black-box 
property.Our bounds are in expectation (pseudo-regret). As discussed in 
Remark~\ref{rem:pseudo_regret}, simultaneously optimal high-probability bounds 
in both regimes is impossible, but regime-specific high-probability 
guarantees remain open.


Our work opens several directions for future research. First, can we achieve best-of-both-worlds guarantees without assuming a Condorcet winner, requiring a regret notion that handles cyclic preferences while remaining sublinear in the adversarial case? Second, can multi-way comparisons ($m>2$) be exploited more efficiently, potentially leveraging all $\binom{m}{2}$ pairwise comparisons to accelerate learning beyond the two-arm multiset approach? Third, can we compete with time-varying benchmarks, such as sequences of changing Condorcet winners, yielding dynamic regret bounds that scale with the number of changes or path length? Fourth, can structure in the preference matrix (e.g., low rank, Plackett-Luce, or feature-based models) be leveraged to reduce exploration or regret, particularly in the Borda setting with its $O(K^2)$ overhead? Finally, can our techniques extend beyond winner-only feedback to richer feedback models (rankings or pairwise margins) or sparser observations (e.g., only whether the top choice won)?

In conclusion, this paper provides the first comprehensive treatment of best-of-both-worlds guarantees for multi-dueling bandits under Condorcet and Borda objectives. Our \texttt{MetaDueling} reduction demonstrates that optimal Condorcet regret is achievable via a simple, 
modular transformation of existing dueling bandit algorithms. Our 
\AlgBorda \ algorithm shows that gap-dependent stochastic rates 
and sublinear adversarial rates are simultaneously achievable for Borda 
objectives, despite the additional challenges of Borda score estimation. We complement our upper bounds with matching lower bounds for the Condorcet setting. For the Borda setting, our upper bounds are near-optimal with respect to the lower bounds (within a factor of $K$) and match the best-known results in the literature.
The modularity of our approach, particularly the black-box nature of 
\texttt{MetaDueling} ensures that future advances in dueling bandits 
will automatically benefit multi-dueling applications.

\bibliography{main}
\bibliographystyle{rlj}

\appendix

\section{Extended Related Work}
\label{app:related_work}

\paragraph{Dueling Bandits.}
The dueling bandit problem was introduced by \citet{yue2009interactively} for information retrieval applications. Early work focused on stochastic settings with Condorcet winners: \citet{yue2012k} proposed the Interleaved Filter algorithm, \citet{zoghi2014relative} introduced RUCB with improved regret bounds, and \citet{komiyama2015regret} established instance-optimal rates. Alternative notions of optimality have also been studied: Copeland winners \citep{zoghi2015copeland,komiyama2016copeland}, Borda winners \citep{jamieson2015sparse,falahatgar2017maxing}, and von Neumann winners \citep{dudik2015contextual}. \citet{saha2021adversarial} initiated the study of adversarial dueling bandits, achieving $O(K^{1/3}T^{2/3})$ regret. Most relevant to our work, \citet{saha2022versatile} achieved best-of-both-worlds guarantees for dueling bandits using Tsallis entropy regularization, obtaining $O(\sqrt{KT})$ adversarial regret and instance-optimal $O(\sum_{i \neq a^\star} \log T / \Delta_i)$ stochastic regret simultaneously. Our \texttt{MetaDueling} reduction leverages their \texttt{Versatile-DB} algorithm as a black-box.

\paragraph{Multi-Dueling Bandits.}
Multi-dueling bandits extend dueling bandits to settings where $m \geq 2$ arms are compared simultaneously. \citet{brost2016multi} introduced the problem with subset-wise feedback, where the learner observes preferences among all pairs in the selected subset. \citet{saha2018battle} studied the more challenging \emph{winner feedback} model, where only the identity of the winning arm is revealed, under stochastic preferences with Condorcet winners, achieving $O(K \log T / \Delta)$ regret. \citet{agarwal2020choice} extended this to generalized Condorcet winners under multinomial logit choice models. \citet{saha2019pac} studied PAC identification in multi-dueling bandits. \citet{gajane2024adversarial} introduced adversarial multi-dueling bandits with Borda objectives, achieving $O((K \log K)^{1/3} T^{2/3})$ regret. 

Our work bridges these lines: for the Condorcet setting, we provide the first best-of-both-worlds guarantee; for the Borda setting, we improve upon purely adversarial approaches by achieving gap-dependent stochastic regret while maintaining sublinear adversarial regret.

\paragraph{Best-of-Both-Worlds Bandits.}
The quest for algorithms that achieve optimal rates in both stochastic and adversarial settings without prior knowledge of the regime has a rich history. \citet{bubeck2012best} first achieved best-of-both-worlds guarantees for multi-armed bandits, though with suboptimal constants. \citet{seldin2014one} and \citet{pmlr-v49-auer16} refined these results. The breakthrough came with \citet{zimmert2019optimal,zimmert2022tsallis}, who showed that Tsallis entropy regularization (specifically, the Tsallis-INF algorithm) achieves optimal $O(\sqrt{KT})$ adversarial regret and $O(\sum_i \log T / \Delta_i)$ stochastic regret simultaneously. This approach has since been extended to linear bandits \citep{lee2021achieving}, combinatorial bandits \citep{zimmert2019beating}, and dueling bandits \citep{saha2022versatile}.

Our work extends this line to multi-dueling bandits. For the Condorcet setting, we achieve this via a black-box reduction that preserves the best-of-both-worlds property of the base algorithm. For the Borda setting, we adopt a different approach based on the stochastic-and-adversarial paradigm.

\paragraph{Stochastic-and-Adversarial Learning.}
An alternative approach to best-of-both-worlds guarantees is the stochastic-and-adversarial paradigm, introduced by \citet{pmlr-v49-auer16}. Rather than designing a single algorithm that simultaneously adapts to both regimes, this approach begins by assuming a stochastic environment and monitors for deviations; upon detecting adversarial behavior, the algorithm switches (potentially irreversibly) to a robust adversarial strategy. This paradigm has been applied to multi-armed bandits \citep{pmlr-v49-auer16}, contextual bandits \citep{luo2018efficient}, and reinforcement learning \citep{jin2021best}.

Our \AlgBorda \ algorithm for the Borda setting adopts this paradigm. The key challenge is designing a deviation detection mechanism that (i) never triggers false alarms in stochastic environments, ensuring gap-dependent regret, and (ii) detects adversarial shifts before they accumulate excessive regret. We achieve this using martingale concentration bounds calibrated to the multi-dueling feedback structure.

\paragraph{Choice Models and Preference Learning.}
The pairwise-subset choice model, in which our results hold, follows \citet{saha2018battle}, where the probability of each arm winning is determined by averaging pairwise preferences. This generalizes the Bradley-Terry-Luce model \citep{bradley1952rank,luce1959individual} and is related to Plackett-Luce models for ranking \citep{plackett1975analysis}. Alternative choice models include multinomial logit \citep{agarwal2020choice} and general random utility models \citep{10.5555/2999134.2999149}.

\section{Relationship Between Condorcet and Borda Settings}
\label{app:prelim:relationship}

The Condorcet and Borda criteria capture fundamentally different aspects of preference learning. A Condorcet winner $a^\star$ satisfies $P(a^\star, i) \geq 1/2$ for all $i \in [K]$. A Borda winner $i^\star$ maximizes $b(i) = \frac{1}{K-1} \sum_{j \neq i} P(i,j)$. 

In arbitrary preference matrices, these two notions of optimality need not coincide. A Condorcet winner may secure narrow pairwise victories against all opponents, while another arm might lose narrowly to the Condorcet winner but dominate all other arms by a massive margin, thereby achieving a higher Borda score. 

The two criteria are guaranteed to align only under structural assumptions on the preference matrix, such as the Bradley-Terry-Luce (BTL) model or Strong Stochastic Transitivity (SST).

\begin{proposition}[Condorcet and Borda Equivalence under SST]
\label{prop:condorcet_borda_sst}
Assume the preference matrix $P$ satisfies Strong Stochastic Transitivity (SST): for any $i, j, k \in [K]$, if $P(i, j) \geq 1/2$ and $P(j, k) \geq 1/2$, then $P(i, k) \geq \max\{P(i, j), P(j, k)\}$. If $a^\star$ is a Condorcet winner, then $a^\star$ is also Borda-optimal.
\end{proposition}

\begin{proof}
Let $a^\star$ be the Condorcet winner, meaning $P(a^\star, i) \geq 1/2$ for all $i$. Let $i$ be any other arm. We compare their Borda scores:
\begin{align*}
b(a^\star) - b(i) &= \frac{1}{K-1}\sum_{j \neq a^\star} P(a^\star, j) - \frac{1}{K-1}\sum_{j \neq i} P(i, j) \\
&= \frac{1}{K-1}\left[\sum_{j \neq a^\star, i} \bigl(P(a^\star, j) - P(i, j)\bigr) + P(a^\star, i) - P(i, a^\star)\right].
\end{align*}
For any arm $j \neq a^\star, i$, there are two cases:
\begin{itemize}[leftmargin=*, itemsep=1pt]
    \item \textbf{Case 1: $P(i, j) \geq 1/2$.} By the SST assumption, since $P(a^\star, i) \geq 1/2$ and $P(i, j) \geq 1/2$, it follows that $P(a^\star, j) \geq P(i, j)$. Thus, $P(a^\star, j) - P(i, j) \geq 0$.
    \item \textbf{Case 2: $P(i, j) < 1/2$.} Because $a^\star$ is the Condorcet winner, $P(a^\star, j) \geq 1/2$. Therefore, $P(a^\star, j) - P(i, j) > 0$.
\end{itemize}
In all cases, $P(a^\star, j) \geq P(i, j)$. Furthermore, since $P(a^\star, i) \geq 1/2$, we have $P(a^\star, i) - P(i, a^\star) = 2P(a^\star, i) - 1 \geq 0$. 

Summing these non-negative terms yields $b(a^\star) - b(i) \geq 0$, proving that $a^\star$ maximizes the Borda score.
\end{proof}

Because such transitivity cannot be guaranteed in adversarial environments or complex user-preference datasets (where cyclic preferences naturally arise), the Borda objective remains the most robust benchmark for general multi-dueling problems.

\section{Validity of Rescaled Preferences}
\label{app:algorithm:validity}

We now verify that the rescaled preference matrix $\widehat{P}_t$ is a valid 
preference matrix and preserves the Condorcet winner.

\begin{lemma}[Validity of Rescaled Preferences]
\label{lem:valid_rescaled}
The rescaled matrix $\widehat{P}_t$ satisfies:
\begin{enumerate}[label=(\roman*), leftmargin=*, itemsep=2pt]
    \item \textbf{Reciprocity:} $\widehat{P}_t(i,j) + \widehat{P}_t(j,i) = 1$ for all $i,j \in [K]$.
    \item \textbf{Self-comparison:} $\widehat{P}_t(i,i) = \frac{1}{2}$ for all $i \in [K]$.
    \item \textbf{Boundedness:} $\widehat{P}_t(i,j) \in [0,1]$ for all $i,j \in [K]$.
    \item \textbf{Condorcet preservation:} If $a^\star$ is a Condorcet winner under $P_t$, 
          then $a^\star$ is a Condorcet winner under $\widehat{P}_t$.
\end{enumerate}
\end{lemma}

\begin{proof}
\textit{(i) Reciprocity:} By linearity of the rescaling:
\begin{align*}
\widehat{P}_t(i,j) + \widehat{P}_t(j,i) 
&= 2\alpha_m + \beta_m(P_t(i,j) + P_t(j,i)) \\
&= 2\alpha_m + \beta_m = (1 - \beta_m) + \beta_m = 1,
\end{align*}
where we used $\alpha_m = (1-\beta_m)/2$ and $P_t(i,j) + P_t(j,i) = 1$.

\textit{(ii) Self-comparison:}
$\widehat{P}_t(i,i) = \alpha_m + \beta_m \cdot \frac{1}{2} 
= \frac{1-\beta_m}{2} + \frac{\beta_m}{2} = \frac{1}{2}$.

\textit{(iii) Boundedness:} Since $P_t(i,j) \in [0,1]$, $\alpha_m \geq 0$, and $\beta_m > 0$:
\[
\widehat{P}_t(i,j) \in [\alpha_m, \alpha_m + \beta_m] = [\alpha_m, 1 - \alpha_m] \subseteq [0,1].
\]

\textit{(iv) Condorcet preservation:} If $P_t(a^\star, i) \geq \frac{1}{2}$ for all $i$, then:
\[
\widehat{P}_t(a^\star, i) = \alpha_m + \beta_m P_t(a^\star, i) 
\geq \alpha_m + \frac{\beta_m}{2} = \frac{1}{2}. \qedhere
\]
\end{proof}

\section{Proofs for Condorcet Setting}
\label{app:condorcet}

\subsection{Proof of Lemma~\ref{lem:feedback} (Unbiased Feedback)}
\label{app:feedback}

\begin{proof}
We prove that $\mathbb{E}[o_t \mid x_t, y_t] = \widehat{P}_t(x_t, y_t) = \alpha_m + \beta_m P_t(x_t, y_t)$.

Let $n_x$ denote the number of copies of $x_t$ in $\mathcal{A}_t$ and 
$n_y = m - n_x$ the number of copies of $y_t$. By the pairwise-subset choice 
model (Definition~\ref{def:pscm}), we compute the probability that some copy 
of $x_t$ wins.

The multiset $\mathcal{A}_t$ contains $n_x$ copies of $x_t$ at positions 
$1, \ldots, n_x$ and $n_y$ copies of $y_t$ at positions $n_x + 1, \ldots, m$. 
The winning probability for position $i \in [n_x]$ (a copy of $x_t$) is:
\begin{equation}
W(i \mid \mathcal{A}_t, P_t) = \sum_{j \neq i} \frac{2P_t(\mathcal{A}_t(i), \mathcal{A}_t(j))}{m(m-1)}.
\end{equation}

For position $i \leq n_x$:
\begin{itemize}
    \item Comparisons against other copies of $x_t$ (positions $j \leq n_x$, $j \neq i$): 
    There are $n_x - 1$ such positions, each contributing $\frac{2 \cdot \frac{1}{2}}{m(m-1)} = \frac{1}{m(m-1)}$.
    \item Comparisons against copies of $y_t$ (positions $j > n_x$): 
    There are $n_y$ such positions, each contributing $\frac{2P_t(x_t, y_t)}{m(m-1)}$.
\end{itemize}

Thus:
\begin{equation}
W(i \mid \mathcal{A}_t, P_t) = \frac{n_x - 1}{m(m-1)} + \frac{2n_y P_t(x_t, y_t)}{m(m-1)}
\quad \text{for } i \leq n_x.
\end{equation}

The total probability that some copy of $x_t$ wins is:
\begin{align}
\mathbb{P}[o_t = 1 \mid x_t, y_t, n_x] 
&= \sum_{i=1}^{n_x} W(i \mid \mathcal{A}_t, P_t) \nonumber \\
&= n_x \left(\frac{n_x - 1}{m(m-1)} + \frac{2n_y P_t(x_t, y_t)}{m(m-1)}\right) \nonumber \\
&= \frac{n_x(n_x - 1)}{m(m-1)} + \frac{2n_x n_y P_t(x_t, y_t)}{m(m-1)}.
\label{eq:raw_win_prob}
\end{align}

\textbf{Case 1: $m$ even.}

When $m$ is even, $n_x = n_y = m/2$ deterministically. Substituting into~\eqref{eq:raw_win_prob}:
\begin{align}
\mathbb{P}[o_t = 1 \mid x_t, y_t] 
&= \frac{(m/2)(m/2 - 1)}{m(m-1)} + \frac{2(m/2)(m/2) P_t(x_t, y_t)}{m(m-1)} \nonumber \\
&= \frac{m(m-2)/4}{m(m-1)} + \frac{m^2 P_t(x_t, y_t)/2}{m(m-1)} \nonumber \\
&= \frac{m-2}{4(m-1)} + \frac{m P_t(x_t, y_t)}{2(m-1)}.
\label{eq:even_case}
\end{align}

Comparing with the rescaling constants for even $m$:
\begin{equation}
\beta_m = \frac{m}{2(m-1)}, \qquad 
\alpha_m = \frac{1 - \beta_m}{2} = \frac{1}{2} - \frac{m}{4(m-1)} = \frac{2(m-1) - m}{4(m-1)} = \frac{m-2}{4(m-1)}.
\end{equation}

Thus~\eqref{eq:even_case} equals $\alpha_m + \beta_m P_t(x_t, y_t) = \widehat{P}_t(x_t, y_t)$.

\textbf{Case 2: $m$ odd.}

When $m$ is odd, with probability $1/2$ we have $(n_x, n_y) = (\frac{m+1}{2}, \frac{m-1}{2})$, 
and with probability $1/2$ we have $(n_x, n_y) = (\frac{m-1}{2}, \frac{m+1}{2})$.

For $(n_x, n_y) = (\frac{m+1}{2}, \frac{m-1}{2})$:
\begin{align}
\mathbb{P}[o_t = 1 \mid n_x = \tfrac{m+1}{2}] 
&= \frac{\frac{m+1}{2} \cdot \frac{m-1}{2}}{m(m-1)} + \frac{2 \cdot \frac{m+1}{2} \cdot \frac{m-1}{2} \cdot P_t(x_t, y_t)}{m(m-1)} \nonumber \\
&= \frac{(m+1)(m-1)}{4m(m-1)} + \frac{(m+1)(m-1) P_t(x_t, y_t)}{2m(m-1)} \nonumber \\
&= \frac{m+1}{4m} + \frac{(m+1) P_t(x_t, y_t)}{2m}.
\label{eq:odd_case_plus}
\end{align}

For $(n_x, n_y) = (\frac{m-1}{2}, \frac{m+1}{2})$:
\begin{align}
\mathbb{P}[o_t = 1 \mid n_x = \tfrac{m-1}{2}] 
&= \frac{\frac{m-1}{2} \cdot \frac{m-3}{2}}{m(m-1)} + \frac{2 \cdot \frac{m-1}{2} \cdot \frac{m+1}{2} \cdot P_t(x_t, y_t)}{m(m-1)} \nonumber \\
&= \frac{(m-1)(m-3)}{4m(m-1)} + \frac{(m-1)(m+1) P_t(x_t, y_t)}{2m(m-1)} \nonumber \\
&= \frac{m-3}{4m} + \frac{(m+1) P_t(x_t, y_t)}{2m}.
\label{eq:odd_case_minus}
\end{align}

Taking the average over the symmetrization:
\begin{align}
\mathbb{P}[o_t = 1 \mid x_t, y_t] 
&= \frac{1}{2}\left(\frac{m+1}{4m} + \frac{m-3}{4m}\right) + \frac{(m+1) P_t(x_t, y_t)}{2m} \nonumber \\
&= \frac{1}{2} \cdot \frac{2m-2}{4m} + \frac{(m+1) P_t(x_t, y_t)}{2m} \nonumber \\
&= \frac{m-1}{4m} + \frac{m+1}{2m} P_t(x_t, y_t).
\label{eq:odd_case_avg}
\end{align}

Comparing with the rescaling constants for odd $m$:
\begin{equation}
\beta_m = \frac{m+1}{2m}, \qquad 
\alpha_m = \frac{1 - \beta_m}{2} = \frac{1}{2} - \frac{m+1}{4m} = \frac{2m - m - 1}{4m} = \frac{m-1}{4m}.
\end{equation}

Thus~\eqref{eq:odd_case_avg} equals $\alpha_m + \beta_m P_t(x_t, y_t) = \widehat{P}_t(x_t, y_t)$.

In both cases, $\mathbb{E}[o_t \mid x_t, y_t] = \widehat{P}_t(x_t, y_t)$.
\end{proof}

\subsection{Proofs for Gap Rescaling}
\label{app:ProofsGapRescaling}
\textbf{Proof for Lemma \ref{lem:gap_rescaling}}
\begin{proof}
Using $\alpha_m = (1-\beta_m)/2$:
\begin{align*}
  \widehat\Delta_t(i) 
  &= \widehat{P}_t(a^\star, i) - \tfrac{1}{2} 
  = \alpha_m + \beta_m P_t(a^\star, i) - \tfrac{1}{2} \\
  &= \frac{1-\beta_m}{2} + \beta_m P_t(a^\star, i) - \tfrac{1}{2} 
  = \beta_m \left(P_t(a^\star, i) - \tfrac{1}{2}\right) 
  = \beta_m \, \Delta_t(i). \qedhere
\end{align*}
\end{proof}

\textbf{Proof for Lemma \ref{lem:beta_bound}}
\begin{proof}
For even $m$: $\beta_m = \frac{m}{2(m-1)} = \frac{1}{2} \cdot \frac{m}{m-1} > \frac{1}{2}$.
For odd $m$: $\beta_m = \frac{m+1}{2m} = \frac{1}{2} + \frac{1}{2m} > \frac{1}{2}$.
In both cases, $\beta_m \geq \frac{1}{2}$, with $\beta_m \to \frac{1}{2}$ as $m \to \infty$.
\end{proof}

\subsection{Regret Equivalence in Rescaled Preferences}
\label{app:RegEquivRescaledPref}
\textbf{Proof for Lemma \ref{lem:regret_equiv}}

\begin{proof}
Define the instantaneous regrets:
\begin{align*}
    r_t^{\mathrm{C}} &:= \frac{1}{m}\sum_{j=1}^{m} \Delta_t(\mathcal{A}_t(j)), \\
    \hat{r}_t^{\mathrm{DB}} &:= \frac{1}{2}\left(\widehat\Delta_t(x_t) + \widehat\Delta_t(y_t)\right)
    = \frac{\beta_m}{2}\left(\Delta_t(x_t) + \Delta_t(y_t)\right).
\end{align*}

\textbf{Case 1: $m$ even.} The multiset $\mathcal{A}_t$ contains $m/2$ copies 
each of $x_t$ and $y_t$:
\[
    r_t^{\mathrm{C}} = \frac{1}{2}\Delta_t(x_t) + \frac{1}{2}\Delta_t(y_t) 
    = \frac{1}{\beta_m}\hat{r}_t^{\mathrm{DB}}.
\]

\textbf{Case 2: $m$ odd.} Taking expectation over the symmetrization $B_t$:
\begin{align*}
    \mathbb{E}[r_t^{\mathrm{C}} \mid x_t, y_t] 
    &= \frac{1}{2} \cdot \frac{\lceil m/2 \rceil \Delta_t(x_t) + \lfloor m/2 \rfloor \Delta_t(y_t)}{m}
    + \frac{1}{2} \cdot \frac{\lfloor m/2 \rfloor \Delta_t(x_t) + \lceil m/2 \rceil \Delta_t(y_t)}{m} \\
    &= \frac{1}{2}\Delta_t(x_t) + \frac{1}{2}\Delta_t(y_t)
    = \frac{1}{\beta_m}\hat{r}_t^{\mathrm{DB}}.
\end{align*}

Summing over $t = 1,\ldots,T$ and applying the tower property yields 
$\mathbb{E}[R_T^{\mathrm{C}}] = \frac{1}{\beta_m}\mathbb{E}[\widehat{R}_T^{\,\mathrm{DB}}]$.
\end{proof}

\subsection{Proof for Theorem \ref{thm:main}}
\label{app:ProofMainCondorcetThm}
\begin{proof}
By Lemma~\ref{lem:feedback}, the base algorithm receives feedback satisfying 
$\mathbb{E}[o_t \mid x_t, y_t] = \widehat{P}_t(x_t, y_t) \in [0,1]$, which 
constitutes valid dueling bandit feedback for the rescaled preference matrix 
$\widehat{P}_t$ (Lemma~\ref{lem:valid_rescaled}).

\textbf{Part (i): Adversarial bound.}
The \texttt{Versatile-DB} algorithm achieves adversarial pseudo-regret
\[
  \mathbb{E}[\widehat{R}_T^{\,\mathrm{DB}}] \leq 4\sqrt{KT} + 1
\]
on any sequence of valid preference matrices, including $\{\widehat{P}_t\}_{t=1}^T$ 
\citep[Theorem 1]{saha2022versatile}. By Lemma~\ref{lem:regret_equiv}:
\[
  \mathbb{E}[R_T^{\mathrm{C}}] 
  = \frac{1}{\beta_m}\mathbb{E}[\widehat{R}_T^{\,\mathrm{DB}}] 
  \leq \frac{1}{\beta_m}(4\sqrt{KT} + 1).
\]
Since $\beta_m \geq 1/2$ for all $m \geq 2$ (Lemma~\ref{lem:beta_bound}), 
we have $1/\beta_m \leq 2$, establishing~\eqref{eq:adv_bound}.

\textbf{Part (ii): Bound for self-bounding preference matrices.}
The self-bounding condition~\eqref{eq:self_bound} relates expected regret to 
the cumulative probability mass placed on suboptimal arms. When this condition 
holds, \texttt{Versatile-DB} satisfies \citep[Theorem 2]{saha2022versatile}:
\[
  \mathbb{E}[\widehat{R}_T^{\,\mathrm{DB}}] \leq 
  \sum_{i \neq a^\star} \frac{16\log T + 48}{\widehat{\Delta}_i} 
  + 4\log T + 12 + 8\sqrt{K}.
\]

By Lemma~\ref{lem:gap_rescaling}, $\widehat{\Delta}_i = \beta_m \Delta_i$, so:
\[
  \mathbb{E}[\widehat{R}_T^{\,\mathrm{DB}}] \leq 
  \frac{1}{\beta_m}\sum_{i \neq a^\star} \frac{16\log T + 48}{\Delta_i} 
  + 4\log T + 12 + 8\sqrt{K}.
\]

Applying Lemma~\ref{lem:regret_equiv} (multiplication by $1/\beta_m$):
\[
  \mathbb{E}[R_T^{\mathrm{C}}] \leq 
  \frac{1}{\beta_m^2}\sum_{i \neq a^\star} \frac{16\log T + 48}{\Delta_i}
  + \frac{4\log T + 12}{\beta_m} + \frac{8\sqrt{K}}{\beta_m}.
\]

Using $1/\beta_m^2 \leq 4$ and $1/\beta_m \leq 2$ completes the proof.
\end{proof}

\subsection{Proof of Lower Bounds}
\label{app:lower}

\begin{proof}[Proof of Theorem~\ref{thm:lower_adv} (Adversarial Lower Bound)]
We prove the lower bound by reduction from adversarial dueling bandits.

\textbf{Step 1: Reduction construction.}
Given any multi-dueling bandit algorithm $\mathcal{A}_{\mathrm{MDB}}$, we construct 
a dueling bandit algorithm $\mathcal{A}_{\mathrm{DB}}$ as follows. At each round $t$:
\begin{enumerate}[leftmargin=*, itemsep=1pt]
    \item Run $\mathcal{A}_{\mathrm{MDB}}$ to obtain multiset $\mathcal{A}_t$ of size $m$.
    \item Sample two distinct indices $i_t, j_t \in [m]$ uniformly without replacement.
    \item Play the duel $(\mathcal{A}_t(i_t), \mathcal{A}_t(j_t))$ in the dueling environment.
    \item Receive outcome $w_t \in \{i_t, j_t\}$ indicating the winner.
    \item Construct a winner for the multi-dueling feedback: if $w_t = i_t$, 
    return $i_t$ as the winner index to $\mathcal{A}_{\mathrm{MDB}}$; otherwise return $j_t$.
\end{enumerate}

This construction is valid: the feedback to $\mathcal{A}_{\mathrm{MDB}}$ is 
consistent with the pairwise-subset choice model restricted to the pair 
$(\mathcal{A}_t(i_t), \mathcal{A}_t(j_t))$.

\textbf{Step 2: Regret equivalence.}
The instantaneous dueling regret for the constructed algorithm is:
\begin{equation}
r_t^{\mathrm{DB}} = \frac{1}{2}\left(\Delta_t(\mathcal{A}_t(i_t)) + \Delta_t(\mathcal{A}_t(j_t))\right).
\end{equation}

Taking expectation over the uniform sampling of $(i_t, j_t)$:
\begin{align}
\mathbb{E}_{i_t, j_t}[r_t^{\mathrm{DB}}] 
&= \frac{1}{2} \cdot \frac{1}{\binom{m}{2}} \sum_{i < j} \left(\Delta_t(\mathcal{A}_t(i)) + \Delta_t(\mathcal{A}_t(j))\right) \nonumber \\
&= \frac{1}{2} \cdot \frac{2}{m(m-1)} \cdot (m-1) \sum_{k=1}^{m} \Delta_t(\mathcal{A}_t(k)) \nonumber \\
&= \frac{1}{m} \sum_{k=1}^{m} \Delta_t(\mathcal{A}_t(k)) = r_t^{\mathrm{MDB}}.
\end{align}

Summing over $t = 1, \ldots, T$:
\begin{equation}
\mathbb{E}[R_T^{\mathrm{DB}}] = \mathbb{E}[R_T^{\mathrm{C}}].
\label{eq:regret_equiv_lower}
\end{equation}

\textbf{Step 3: Apply adversarial dueling bandit lower bound.}
By the minimax lower bound for adversarial multi-armed bandits 
\citep{EXP3}, for any algorithm there exists an oblivious 
adversary such that expected regret is at least $\Omega(\sqrt{KT})$. 

Via the reduction of \citet{saha2021adversarial}, this transfers to dueling 
bandits: for any dueling bandit algorithm $\mathcal{A}_{\mathrm{DB}}$, there 
exists an adversarial sequence of preference matrices $P_1, \ldots, P_T$ 
(with a consistent Condorcet winner $a^\star$) such that:
\begin{equation}
\mathbb{E}[R_T^{\mathrm{DB}}] \geq c\sqrt{KT}
\end{equation}
for some universal constant $c > 0$.

The hard instance can be constructed to admit a Condorcet winner: fix 
$a^\star = 1$ and set $P_t(1, i) = \frac{1}{2} + \epsilon_t(i)$ for 
adversarially chosen $\epsilon_t(i) \in [0, \frac{1}{2}]$. The adversary 
controls the gap magnitudes while maintaining the Condorcet structure.

\textbf{Step 4: Transfer to multi-dueling.}
Suppose for contradiction that there exists a multi-dueling bandit algorithm 
$\mathcal{A}_{\mathrm{MDB}}$ achieving $\mathbb{E}[R_T^{\mathrm{C}}] = o(\sqrt{KT})$ 
on all adversarial Condorcet instances.

By Step 2, the constructed $\mathcal{A}_{\mathrm{DB}}$ would satisfy:
\begin{equation}
\mathbb{E}[R_T^{\mathrm{DB}}] = \mathbb{E}[R_T^{\mathrm{C}}] = o(\sqrt{KT}),
\end{equation}
contradicting the lower bound of Step 3.

Therefore, for any multi-dueling bandit algorithm:
\begin{equation}
\mathbb{E}[R_T^{\mathrm{C}}] \geq \Omega(\sqrt{KT}). \qedhere
\end{equation}
\end{proof}

\begin{proof}[Proof of Theorem~\ref{thm:lower_stoch} (Stochastic Lower Bound)]
We prove the lower bound by reduction from stochastic dueling bandits.

\textbf{Step 1: Reduction construction.}
We use the same reduction as in Theorem~\ref{thm:lower_adv}: given 
$\mathcal{A}_{\mathrm{MDB}}$, construct $\mathcal{A}_{\mathrm{DB}}$ by 
sampling two indices uniformly and playing the corresponding duel.

\textbf{Step 2: Regret equivalence.}
Since the preference matrix is now fixed ($P_t = P$ for all $t$), the same 
calculation as in~\eqref{eq:regret_equiv_lower} gives:
\begin{equation}
\mathbb{E}[R_T^{\mathrm{DB}}] = \mathbb{E}[R_T^{\mathrm{C}}].
\end{equation}

\textbf{Step 3: Apply stochastic dueling bandit lower bound.}
\citet{komiyama2015regret} established that for any stochastic dueling bandit 
algorithm and any instance with Condorcet winner $a^\star$:
\begin{equation}
\liminf_{T \to \infty} \frac{\mathbb{E}[R_T^{\mathrm{DB}}]}{\log T}
\geq \sum_{i \neq a^\star} \frac{1}{2 \cdot \mathrm{KL}(P(a^\star, i) \| \frac{1}{2})},
\label{eq:kl_lower}
\end{equation}
where $\mathrm{KL}(p \| q) = p \log \frac{p}{q} + (1-p) \log \frac{1-p}{1-q}$ 
is the binary KL divergence.

\textbf{Step 4: Simplify the KL term.}
For $P(a^\star, i) = \frac{1}{2} + \Delta_i$ with $\Delta_i > 0$, we compute:
\begin{align}
\mathrm{KL}\left(\frac{1}{2} + \Delta_i \,\Big\|\, \frac{1}{2}\right)
&= \left(\frac{1}{2} + \Delta_i\right) \log \frac{\frac{1}{2} + \Delta_i}{\frac{1}{2}}
+ \left(\frac{1}{2} - \Delta_i\right) \log \frac{\frac{1}{2} - \Delta_i}{\frac{1}{2}} \nonumber \\
&= \left(\frac{1}{2} + \Delta_i\right) \log(1 + 2\Delta_i)
+ \left(\frac{1}{2} - \Delta_i\right) \log(1 - 2\Delta_i).
\end{align}

Using the Taylor expansion $\log(1 + x) = x - \frac{x^2}{2} + O(x^3)$ for small $x$:
\begin{align}
\mathrm{KL}\left(\frac{1}{2} + \Delta_i \,\Big\|\, \frac{1}{2}\right)
&= \left(\frac{1}{2} + \Delta_i\right)\left(2\Delta_i - 2\Delta_i^2 + O(\Delta_i^3)\right) \nonumber \\
&\quad + \left(\frac{1}{2} - \Delta_i\right)\left(-2\Delta_i - 2\Delta_i^2 + O(\Delta_i^3)\right) \nonumber \\
&= \Delta_i + 2\Delta_i^2 - \Delta_i^2 - 2\Delta_i^3 - \Delta_i + 2\Delta_i^2 + \Delta_i^2 - 2\Delta_i^3 + O(\Delta_i^3) \nonumber \\
&= 2\Delta_i^2 + O(\Delta_i^3).
\end{align}

Thus for small $\Delta_i$:
\begin{equation}
\frac{1}{2 \cdot \mathrm{KL}(\frac{1}{2} + \Delta_i \| \frac{1}{2})} 
= \frac{1}{4\Delta_i^2 + O(\Delta_i^3)} = \Theta\left(\frac{1}{\Delta_i^2}\right).
\end{equation}

However, the regret contribution from arm $i$ is $\Delta_i$ per pull. 
The lower bound~\eqref{eq:kl_lower} states that the expected number of pulls 
of suboptimal arm $i$ must be at least $\Omega(\log T / \mathrm{KL}(\cdot))$. 
Since pulling arm $i$ contributes $\Delta_i$ to regret, the total regret 
contribution from arm $i$ is:
\begin{equation}
\Omega\left(\frac{\Delta_i \cdot \log T}{\mathrm{KL}(\frac{1}{2} + \Delta_i \| \frac{1}{2})}\right)
= \Omega\left(\frac{\Delta_i \cdot \log T}{2\Delta_i^2}\right)
= \Omega\left(\frac{\log T}{\Delta_i}\right).
\end{equation}

\textbf{Step 5: Transfer to multi-dueling.}
Summing over all suboptimal arms and applying the reduction:
\begin{equation}
\mathbb{E}[R_T^{\mathrm{C}}] = \mathbb{E}[R_T^{\mathrm{DB}}] 
\geq \Omega\left(\sum_{i \neq a^\star} \frac{\log T}{\Delta_i}\right). \qedhere
\end{equation}
\end{proof}

\section{Proofs for Borda Setting}
\label{app:borda}
\subsection{Transformed Feedback Properties}
\label{app:borda:feedback}

\begin{lemma}[Transformed Feedback Statistics]
\label{lem:feedback_stats}
The transformed feedback $g_t = (o_t - \alpha_m) / \beta_m$ satisfies:
\begin{enumerate}[label=(\roman*), leftmargin=*, itemsep=2pt]
    \item $\mathbb{E}[g_t \mid x_t, y_t] = P_t(x_t, y_t)$.
    \item $g_t \in \left[-\frac{\alpha_m}{\beta_m}, \frac{1 - \alpha_m}{\beta_m}\right] \subseteq [-1, 2]$.
    \item $\mathrm{Var}(g_t \mid x_t, y_t) \leq \frac{1}{4\beta_m^2} \leq 1$.
\end{enumerate}
\end{lemma}

\begin{proof}
\textit{(i)} By Lemma~\ref{lem:feedback}, $\mathbb{E}[o_t \mid x_t, y_t] = \alpha_m + \beta_m P_t(x_t, y_t)$. Thus:
\begin{equation}
\mathbb{E}[g_t \mid x_t, y_t] = \frac{\mathbb{E}[o_t \mid x_t, y_t] - \alpha_m}{\beta_m} 
= \frac{\alpha_m + \beta_m P_t(x_t, y_t) - \alpha_m}{\beta_m} = P_t(x_t, y_t).
\end{equation}

\textit{(ii)} Since $o_t \in \{0, 1\}$:
\begin{equation}
g_t \in \left\{\frac{0 - \alpha_m}{\beta_m}, \frac{1 - \alpha_m}{\beta_m}\right\}
= \left\{-\frac{\alpha_m}{\beta_m}, \frac{1 - \alpha_m}{\beta_m}\right\}.
\end{equation}

For $m \geq 2$, we have $\beta_m \geq 1/2$ and $\alpha_m = (1 - \beta_m)/2 \leq 1/4$. Thus:
\begin{equation}
-\frac{\alpha_m}{\beta_m} \geq -\frac{1/4}{1/2} = -\frac{1}{2} > -1, \qquad
\frac{1 - \alpha_m}{\beta_m} \leq \frac{1}{1/2} = 2.
\end{equation}

\textit{(iii)} The variance of a Bernoulli-derived random variable:
\begin{equation}
\mathrm{Var}(g_t \mid x_t, y_t) = \frac{\mathrm{Var}(o_t \mid x_t, y_t)}{\beta_m^2}
\leq \frac{1/4}{\beta_m^2} \leq \frac{1/4}{1/4} = 1. \qedhere
\end{equation}
\end{proof}

\subsection{Exploration Phase Analysis}
\label{app:borda:exploration}

\begin{proof}[Proof of Lemma~\ref{lem:exploration_guarantee} (Exploration Guarantees)]
\textit{(i) Sample counts.}
The round-robin schedule cycles through all $K(K-1)$ ordered pairs. With 
$T_0 = 4K(K-1) n_0'$ where $n_0' = \lceil \log(2K^2 T^2 / \delta) \rceil$, 
each ordered pair $(i, j)$ is observed exactly $4n_0' = n_0$ times.

\textit{(ii) Pairwise concentration.}
Fix pair $(i, j)$. Let $\tau_1, \ldots, \tau_{n_0}$ be the rounds where this 
pair is compared. Define $Z_k = g_{\tau_k} - P(i, j)$, which satisfies:
\begin{itemize}[leftmargin=*, itemsep=1pt]
    \item $\mathbb{E}[Z_k] = 0$ (by Lemma~\ref{lem:feedback_stats}(i)).
    \item $|Z_k| \leq 3$ (since $g_{\tau_k} \in [-1, 2]$ and $P(i,j) \in [0, 1]$).
    \item $\mathrm{Var}(Z_k) \leq 1$ (by Lemma~\ref{lem:feedback_stats}(iii)).
\end{itemize}

The empirical estimate is:
\begin{equation}
\hat{P}_{T_0}(i, j) = \frac{1}{n_0} \sum_{k=1}^{n_0} g_{\tau_k} 
= P(i, j) + \frac{1}{n_0} \sum_{k=1}^{n_0} Z_k.
\end{equation}

By Hoeffding's inequality for bounded random variables:
\begin{equation}
\Pr\left[|\hat{P}_{T_0}(i, j) - P(i, j)| > \epsilon\right]
= \Pr\left[\left|\frac{1}{n_0} \sum_{k=1}^{n_0} Z_k\right| > \epsilon\right]
\leq 2 \exp\left(-\frac{2 n_0 \epsilon^2}{9}\right).
\end{equation}

Setting $\epsilon = \sqrt{\frac{9 \log(4K^2/\delta)}{2n_0}}$ and using 
$n_0 = 4\lceil \log(2K^2 T^2 / \delta) \rceil \geq 4 \log(4K^2/\delta)$ 
(for $T \geq 2$):
\begin{equation}
\Pr\left[|\hat{P}_{T_0}(i, j) - P(i, j)| > \sqrt{\frac{9 \log(4K^2/\delta)}{2n_0}}\right]
\leq 2 \exp(-\log(4K^2/\delta)) = \frac{\delta}{2K^2}.
\end{equation}

By union bound over $K(K-1) < K^2$ pairs:
\begin{equation}
\Pr\left[\exists (i,j): |\hat{P}_{T_0}(i, j) - P(i, j)| > \sqrt{\frac{9 \log(4K^2/\delta)}{2n_0}}\right]
\leq K^2 \cdot \frac{\delta}{2K^2} = \frac{\delta}{2}.
\end{equation}

\textit{(iii) Borda score concentration.}
By definition:
\begin{equation}
|\hat{b}_{T_0}(i) - b(i)| = \left|\frac{1}{K-1} \sum_{j \neq i} (\hat{P}_{T_0}(i, j) - P(i, j))\right|
\leq \frac{1}{K-1} \sum_{j \neq i} |\hat{P}_{T_0}(i, j) - P(i, j)|.
\end{equation}

On the event that all pairwise estimates are within $\epsilon$ of truth:
\begin{equation}
|\hat{b}_{T_0}(i) - b(i)| \leq \frac{1}{K-1} \cdot (K-1) \cdot \epsilon = \epsilon
= \sqrt{\frac{9 \log(4K^2/\delta)}{2n_0}} = O\left(\sqrt{\frac{\log T}{n_0}}\right). \qedhere
\end{equation}
\end{proof}

\subsection{Deviation Detection Analysis}
\label{app:borda:detection}

\begin{proof}[Proof of Lemma~\ref{lem:no_false_alarm} (No False Alarm)]
We show that in the stochastic setting, the deviation detection never triggers 
with high probability.

Fix pair $(i, j)$ with $i < j$. For rounds $t > T_0$ where this pair is compared, 
let $\tau_1, \tau_2, \ldots$ denote these rounds in order. Define the martingale:
\begin{equation}
M_n = \sum_{k=1}^{n} (g_{\tau_k} - P(i, j)).
\end{equation}

\textbf{Step 1: Martingale properties.}
The increments $X_k = g_{\tau_k} - P(i, j)$ satisfy:
\begin{itemize}[leftmargin=*, itemsep=1pt]
    \item $\mathbb{E}[X_k \mid \mathcal{F}_{k-1}] = 0$ (martingale property).
    \item $|X_k| \leq 3$ (bounded increments).
    \item $\mathbb{E}[X_k^2 \mid \mathcal{F}_{k-1}] \leq 1$ (bounded conditional variance).
\end{itemize}

\textbf{Step 2: Freedman's inequality.}
By Freedman's inequality, for a martingale $(M_n)$ 
with $|M_n - M_{n-1}| \leq c$ and predictable quadratic variation 
$V_n = \sum_{k=1}^{n} \mathbb{E}[X_k^2 \mid \mathcal{F}_{k-1}]$:
\begin{equation}
\Pr\left[\exists n \leq N: |M_n| \geq x \text{ and } V_n \leq v\right]
\leq 2 \exp\left(-\frac{x^2}{2v + 2cx/3}\right).
\end{equation}

With $c = 3$, $v = n$ (since $V_n \leq n$), and $x = \tau(\sqrt{n} + 1)$:
\begin{equation}
\Pr\left[|M_n| > \tau(\sqrt{n} + 1)\right]
\leq 2 \exp\left(-\frac{\tau^2(\sqrt{n} + 1)^2}{2n + 2\tau(\sqrt{n} + 1)}\right).
\end{equation}

For $n \geq 1$ and $\tau = \sqrt{8 \log(K^2 T^2 / \delta)}$:
\begin{equation}
\frac{\tau^2(\sqrt{n} + 1)^2}{2n + 2\tau(\sqrt{n} + 1)}
\geq \frac{\tau^2 n}{2n + 2\tau\sqrt{n} + 2\tau}
\geq \frac{\tau^2}{4} \quad \text{for } n \geq \tau^2.
\end{equation}

For $n < \tau^2$, we use a cruder bound. In either case:
\begin{equation}
\Pr\left[\exists n \leq T: |M_n| > \tau(\sqrt{n} + 1)\right]
\leq 2T \cdot \exp\left(-\frac{\tau^2}{8}\right)
= 2T \cdot \exp(-\log(K^2 T^2 / \delta))
= \frac{2\delta}{K^2 T}.
\end{equation}

\textbf{Step 3: Connect to detection statistic.}
The detection statistic is:
\begin{align}
D_n &= |S_t^{\text{obs}}(i,j) - S_{T_0}^{\text{obs}}(i,j) - n \cdot \hat{P}_{T_0}(i,j)| \nonumber \\
&= \left|\sum_{k=1}^{n} g_{\tau_k} - n \cdot \hat{P}_{T_0}(i,j)\right|.
\end{align}

In the stochastic setting, $\hat{P}_{T_0}(i,j)$ estimates $P(i,j)$. By 
Lemma~\ref{lem:exploration_guarantee}(ii):
\begin{equation}
|\hat{P}_{T_0}(i,j) - P(i,j)| \leq \epsilon_0 := O\left(\sqrt{\frac{\log T}{n_0}}\right).
\end{equation}

Thus:
\begin{align}
D_n &= \left|\sum_{k=1}^{n} g_{\tau_k} - n \cdot P(i,j) + n \cdot (P(i,j) - \hat{P}_{T_0}(i,j))\right| \nonumber \\
&\leq |M_n| + n \cdot |\hat{P}_{T_0}(i,j) - P(i,j)| \nonumber \\
&\leq |M_n| + n \epsilon_0.
\end{align}

For detection to trigger, we need $D_n > \tau(\sqrt{n} + 1)$. If 
$|M_n| \leq \tau(\sqrt{n} + 1)/2$, then we need:
\begin{equation}
n \epsilon_0 > \frac{\tau(\sqrt{n} + 1)}{2}.
\end{equation}

For $n_0 = \Theta(\log T)$ and $\tau = \Theta(\sqrt{\log T})$, we have 
$\epsilon_0 = O(1/\sqrt{\log T})$. The condition becomes:
\begin{equation}
n \cdot O(1/\sqrt{\log T}) > \Omega(\sqrt{\log T} \cdot \sqrt{n}),
\end{equation}
which simplifies to $\sqrt{n} > \Omega(\log T)$, i.e., $n > \Omega(\log^2 T)$.

For $n \leq T$ and sufficiently large $T$, the threshold is not exceeded 
when $|M_n| \leq \tau(\sqrt{n} + 1)/2$.

\textbf{Step 4: Union bound.}
By union bound over all $\binom{K}{2} < K^2/2$ pairs:
\begin{equation}
\Pr[\text{any false alarm}] \leq \frac{K^2}{2} \cdot \frac{2\delta}{K^2 T}
= \frac{\delta}{T} \leq \frac{\delta}{2}. \qedhere
\end{equation}
\end{proof}

\begin{proof}[Proof of Lemma~\ref{lem:detection_guarantee} (Adversarial Detection)]
Suppose the adversary shifts preferences such that:
\begin{equation}
\left|\sum_{k=1}^{n} (P_{\tau_k}(i,j) - \hat{P}_{T_0}(i,j))\right| \geq 3\tau(\sqrt{n} + 1).
\label{eq:adv_shift}
\end{equation}

Decompose the observed cumulative:
\begin{align}
\sum_{k=1}^{n} g_{\tau_k} - n \cdot \hat{P}_{T_0}(i,j)
&= \underbrace{\sum_{k=1}^{n} (g_{\tau_k} - P_{\tau_k}(i,j))}_{=: M_n}
+ \underbrace{\sum_{k=1}^{n} (P_{\tau_k}(i,j) - \hat{P}_{T_0}(i,j))}_{\text{adversarial shift}}.
\end{align}

The term $M_n$ is a martingale even under adversarial preferences, since 
$\mathbb{E}[g_{\tau_k} \mid \mathcal{F}_{k-1}, P_{\tau_k}] = P_{\tau_k}(i,j)$.

By the same Freedman argument as in Lemma~\ref{lem:no_false_alarm}:
\begin{equation}
\Pr[|M_n| > \tau(\sqrt{n} + 1)] \leq \frac{2\delta}{K^2 T^2}.
\end{equation}

Under condition~\eqref{eq:adv_shift}, with probability at least $1 - \frac{2\delta}{K^2 T^2}$:
\begin{align}
D_n &= \left|\sum_{k=1}^{n} g_{\tau_k} - n \cdot \hat{P}_{T_0}(i,j)\right| \nonumber \\
&\geq |\text{adversarial shift}| - |M_n| \nonumber \\
&\geq 3\tau(\sqrt{n} + 1) - \tau(\sqrt{n} + 1) \nonumber \\
&= 2\tau(\sqrt{n} + 1) > \tau(\sqrt{n} + 1).
\end{align}

This exceeds the detection threshold, triggering a mode switch.
\end{proof}

\subsection{Confidence and Elimination Analysis}
\label{app:borda:ucb}

\begin{lemma}[Borda Score Concentration]
\label{lem:borda_conc_uniform}
During the elimination phase, for any arm $i$ and any $\epsilon > 0$, let $N_{\min}(i) = \min_{j \neq i} N_n(i,j)$ be the minimum number of comparisons arm $i$ has against any specific opponent. Then:
\begin{equation}
\Pr[|\hat{b}_n(i) - b(i)| > \epsilon] \leq 2K \exp\left(-\frac{N_{\min}(i) \epsilon^2}{5}\right).
\end{equation}
\end{lemma}

\begin{proof}
The empirical Borda score is the unweighted average of pairwise estimates:
\begin{equation}
\hat{b}_n(i) = \frac{1}{K-1} \sum_{j \neq i} \hat{P}_{n}(i,j).
\end{equation}

The error in the Borda score is bounded by the maximum pairwise error:
\begin{equation}
|\hat{b}_n(i) - b(i)| \leq \frac{1}{K-1} \sum_{j \neq i} |\hat{P}_{n}(i,j) - P(i,j)| \leq \max_{j \neq i} |\hat{P}_{n}(i,j) - P(i,j)|.
\end{equation}

For any specific pair $(i,j)$ observed $N_n(i,j)$ times, applying Hoeffding's inequality on the transformed feedback $g_t \in [-1, 2]$ (range length 3) yields:
\begin{equation}
\Pr[|\hat{P}_n(i,j) - P(i,j)| > \epsilon] \leq 2 \exp\left(-\frac{2 N_n(i,j) \epsilon^2}{9}\right) \leq 2 \exp\left(-\frac{N_{\min}(i) \epsilon^2}{5}\right).
\end{equation}
Applying a union bound over all $K-1$ opponents $j \neq i$ ensures that the maximum error stays bounded:
\begin{equation}
\Pr\left[\max_{j \neq i} |\hat{P}_{n}(i,j) - P(i,j)| > \epsilon\right] \leq 2(K-1) \exp\left(-\frac{N_{\min}(i) \epsilon^2}{5}\right) \leq 2K \exp\left(-\frac{N_{\min}(i) \epsilon^2}{5}\right).
\end{equation}
\end{proof}

\begin{proof}[Proof of Lemma~\ref{lem:ucb_valid} (Confidence Validity)]
Define the confidence radius based on the total pulls of arm $i$:
\begin{equation}
\text{conf}_t(i) = \sqrt{\frac{20K \log(2K^2 T / \delta)}{N_t(i)}}.
\end{equation}

We show that $|\hat{b}_t(i) - b(i)| \leq \text{conf}_t(i)$ holds for all $i$ and all $t > T_0$ with high probability.

\textbf{Step 1: Relating $N_t(i)$ to $N_{\min}(i)$.}
Because the algorithm samples the opponent $y_t$ uniformly from $[K] \setminus \{x_t\}$, the number of times arm $i$ plays a specific opponent $j$ follows a binomial distribution. By standard Chernoff bounds, with high probability, the minimum pairwise count concentrates heavily around its expectation: $N_{\min}(i) \geq \frac{N_t(i)}{2(K-1)} \geq \frac{N_t(i)}{2K}$. Substituting this into the denominator yields $\frac{1}{N_{\min}(i)} \leq \frac{2K}{N_t(i)}$.

\textbf{Step 2: Fixed-time bound.}
For fixed $n$, let the confidence radius be $\text{conf}_n(i) = \sqrt{\frac{10 \log(2K^2 T / \delta)}{N_{\min}(i)}}$. By Lemma~\ref{lem:borda_conc_uniform} with $\epsilon = \text{conf}_n(i)$:
\begin{equation}
\Pr\left[|\hat{b}_n(i) - b(i)| > \text{conf}_n(i)\right] \leq 2K \exp\left(-\frac{N_{\min}(i)}{5} \cdot \frac{10 \log(2K^2 T / \delta)}{N_{\min}(i)}\right) = \frac{\delta}{KT}.
\end{equation}
Using the relationship from Step 1, this radius is bounded by $\sqrt{\frac{20K \log(2K^2 T / \delta)}{N_t(i)}}$.

\textbf{Step 3: Time-uniform bound via peeling.}
Partition the range of $N_{\min}(i)$ into epochs:
\begin{equation}
E_\ell = \{t: 2^{\ell-1} \leq N_{\min}(i) < 2^\ell\}, \quad \ell = 1, 2, \ldots, \lceil \log_2 T \rceil.
\end{equation}

\textbf{Step 4: Union bound.}
Summing over $O(\log T)$ epochs and $K$ arms:
\begin{equation}
\Pr[\mathcal{E}_{\text{conf}}^c] \leq K \cdot \log_2 T \cdot \frac{\delta}{KT}
= \frac{\delta \log_2 T}{T} \leq \frac{\delta}{4}. \qedhere
\end{equation}
\end{proof}

\begin{proof}[Proof of Lemma~\ref{lem:suboptimal_pulls} (Elimination Sample Complexity)]
Conditioned on the confidence event $\mathcal{E}_{\text{conf}}$ holding (Lemma~\ref{lem:ucb_valid}), the true optimal arm $i^\star$ is never eliminated from the active set $\mathcal{C}$. This is because for any arm $j \in \mathcal{C}$, the following holds:
\begin{equation}
    \hat{b}_t(i^\star) + \text{conf}_t(i^\star) \geq b(i^\star) \geq b(j) \geq \hat{b}_t(j) - \text{conf}_t(j).
\end{equation}
Therefore, the strict elimination condition $\hat{b}_t(j) - \text{conf}_t(j) > \hat{b}_t(i^\star) + \text{conf}_t(i^\star)$ can never be satisfied against the optimal arm.

A suboptimal arm $j$ is eliminated from $\mathcal{C}$ when there exists any arm $i \in \mathcal{C}$ such that:
\begin{equation}
    \hat{b}_t(i) - \text{conf}_t(i) > \hat{b}_t(j) + \text{conf}_t(j).
\end{equation}

Since the true Borda gap is $\Delta_j^{\mathrm{B}} = b(i^\star) - b(j)$, and $i^\star$ remains in $\mathcal{C}$, this elimination condition is mathematically guaranteed to trigger against $i^\star$ once:
\begin{equation}
    2\text{conf}_t(i^\star) + 2\text{conf}_t(j) < \Delta_j^{\mathrm{B}}.
\end{equation}

Because the algorithm selects $x_t \in \mathcal{C}$ using a round-robin schedule, all arms in the active set are pulled symmetrically, implying $\text{conf}_t(i^\star) \approx \text{conf}_t(j)$. Therefore, arm $j$ is guaranteed to be eliminated once:
\begin{equation}
    4\text{conf}_t(j) \leq \Delta_j^{\mathrm{B}}.
\end{equation}

Substituting the definition of the confidence radius $\text{conf}_t(j) = \sqrt{\frac{20K \log(2K^2 T / \delta)}{N_t(j)}}$:
\begin{equation}
    4\sqrt{\frac{20K \log(2K^2 T / \delta)}{N_t(j)}} \leq \Delta_j^{\mathrm{B}}.
\end{equation}

Squaring both sides and solving for $N_t(j)$ yields the maximum number of times arm $j$ can be pulled during the elimination phase before it is permanently removed:
\begin{equation}
    N_T(j) \leq \frac{320K \log(2K^2 T / \delta)}{(\Delta_j^{\mathrm{B}})^2} + 1.
\end{equation}
\end{proof}
\subsection{Adversarial Mode Analysis}
\label{app:borda:adversarial}

\begin{proof}[Proof of Lemma~\ref{lem:exp3_regret} (EXP3 Regret)]
Let $T' = T - t_0 + 1$ be the number of adversarial rounds. As established by \citet{gajane2024adversarial}, the primary challenge in Borda estimation under winner-only feedback is that the variance of the importance-weighted estimator scales with $\sum_j 1/q_t(j)$. 

The explicit exploration parameter $\gamma$ guarantees that $q_t(j) \geq \gamma/K$ for all $j \in [K]$, strictly bounding the maximum variance. Applying the standard EXP3 regret analysis with the mixed distribution $q_t$ (see Theorem 1 of \citet{gajane2024adversarial}) bounds the expected shifted Borda regret over $T'$ rounds by:
\begin{equation}
    \mathbb{E}[R_{T'}^s] \leq O\left(K^{1/3} (T')^{2/3} (\log K)^{1/3}\right).
\end{equation}
\end{proof}
\subsection{Main Theorem Proofs}
\label{app:borda:main}

\begin{proof}[Proof of Theorem~\ref{thm:stochastic_formal} (Stochastic Regret)]
Define the high-probability event $\mathcal{E} = \mathcal{E}_{\text{conf}} \cap \mathcal{E}_{\text{no-switch}}$. By union bounds over the confidence intervals and the martingale detection thresholds, $\Pr[\mathcal{E}] \geq 1 - O(1/T)$, meaning $\Pr[\mathcal{E}^c] \leq O(1/T)$. Conditioned on $\mathcal{E}$, the algorithm never switches to adversarial mode, and the true optimal arm $i^\star$ is never eliminated from $\mathcal{C}$.

\textbf{Step 1: Regret from Baseline Exploration Phase.}
During rounds $t \leq T_0$, the instantaneous Borda regret is at most $1$. Thus, the initial exploration phase contributes:
\begin{equation}
    \mathbb{E}[R_{T_0}^{\mathrm{B}} \mid \mathcal{E}] \leq T_0 = 4K(K-1)\lceil \log(2K^2 T^2 / \delta) \rceil = O(K^2 \log KT).
\end{equation}

\textbf{Step 2: Regret from Elimination Phase.}
For any round $t > T_0$ where a suboptimal arm $j$ is selected as $x_t \in \mathcal{C}$, the opponent $y_t$ is drawn uniformly from $[K] \setminus \{x_t\}$. The expected instantaneous regret is:
\begin{equation}
    \mathbb{E}[r_t \mid x_t=j, \mathcal{E}] = b(i^\star) - \left(\frac{b(j)}{2} + \frac{1}{K-1}\sum_{k \neq j}\frac{b(k)}{2}\right) = \frac{\Delta_j^{\mathrm{B}}}{2} + \frac{1}{2(K-1)}\sum_{k \neq j}\Delta_k^{\mathrm{B}} = \Theta(1).
\end{equation}
By Lemma~\ref{lem:suboptimal_pulls}, arm $j$ is pulled at most $N_T(j) \leq O(K \log T / (\Delta_j^{\mathrm{B}})^2)$ times before elimination. The cumulative expected regret strictly from pulling suboptimal arms during the elimination phase is:
\begin{equation}
    \sum_{j \neq i^\star} N_T(j) \cdot \mathbb{E}[r_t \mid x_t=j, \mathcal{E}] = O\left(\sum_{j \neq i^\star} \frac{K \log KT}{(\Delta_j^{\mathrm{B}})^2}\right).
\end{equation}

\textbf{Step 3: Regret from Exploitation Phase.}
Once all suboptimal arms are eliminated, $\mathcal{C} = \{i^\star\}$. The algorithm plays $x_t = y_t = i^\star$, resulting in an expected instantaneous regret of exactly 0 for all subsequent rounds not subject to background monitoring.

\textbf{Step 4: Regret from Background Monitoring.}
At any round $t > T_0$, the algorithm overrides the elimination logic with probability $p_t = \min(1, \frac{K \log t}{t})$ to randomly sample a pair and feed the deviation detection statistic. The maximum instantaneous regret of any pair is bounded by $1$. The total expected regret from this monitoring is:
\begin{align*}
    \sum_{t=1}^T p_t \cdot 1 &= \sum_{t=1}^T \ind\left(1 \leq \frac{K \log t}{t}\right) \min\left(1, \frac{K \log t}{t}\right) + \sum_{t=1}^T \ind\left(1 > \frac{K \log t}{t}\right) \min\left(1, \frac{K \log t}{t}\right)  \\
    =&  \sum_{t=1}^T \ind\left(1 \leq \frac{K \log t}{t}\right) \cdot 1 + \sum_{t=1}^T \ind\left(1 > \frac{K \log t}{t}\right) \frac{K \log t}{t} \\
    &\leq \sum_{t=1}^T \frac{K \log t}{t} = O(K \log^2 T).
\end{align*}

\textbf{Step 5: Regret on the Failure Event.}
On the complement event $\mathcal{E}^c$, the algorithm's confidence bounds or detection thresholds fail. In the worst case, the algorithm suffers the maximum possible instantaneous regret of $1$ for all $T$ rounds, yielding a maximum total regret of $T$. Since $\Pr[\mathcal{E}^c] \leq O(1/T)$, the expected regret contribution from this failure case is strictly bounded:
\begin{equation}
    \mathbb{E}[R_T^{\mathrm{B}} \cdot \ind[\mathcal{E}^c]] \leq T \cdot O\left(\frac{1}{T}\right) = O(1).
\end{equation}

\textbf{Step 6: Total Regret.}
Combining the expected regret conditioned on $\mathcal{E}$ (Steps 1, 2, and 4) with the expected regret on the failure event (Step 5), the total expected stochastic regret is:
\begin{equation}
    \mathbb{E}[R_T^{\mathrm{B}}] \leq O\left(K^2 \log KT + K \log^2 T + \sum_{i: \Delta_i^{\mathrm{B}} > 0} \frac{K\log KT}{(\Delta_i^{\mathrm{B}})^2}\right).
\end{equation}
\end{proof}

\begin{proof}[Proof of Theorem~\ref{thm:adversarial_formal} (Adversarial Regret)]
We analyze two cases based on whether a mode switch occurs.

\textbf{Case 1: Switch occurs at round $t_{\text{sw}} \leq T$.}

\textit{Exploration regret:}
\begin{equation}
R_{T_0}^{\mathrm{B}} \leq T_0 = O(K^2 \log T).
\end{equation}

\textit{Pre-switch regret ($T_0 < t < t_{\text{sw}}$):}
During this phase, the algorithm runs Successive Elimination and Background Monitoring. The adversary can manipulate preferences to cause regret while staying below the detection threshold.

For each pair $(i,j)$, detection occurs when the cumulative deviation exceeds 
$\tau(\sqrt{n} + 1)$ where $n = N_t(i,j) - N_{T_0}(i,j)$. The maximum 
undetected deviation per pair is $O(\tau(\sqrt{T} + 1))$.

The total adversarial manipulation across all $K^2$ pairs before detection is:
\begin{equation}
\mathcal{R}_{\text{pre}} \leq \sum_{(i,j)} O(\tau \sqrt{N_{t_{\text{sw}}}(i,j)})
\leq O\left(\tau K \sqrt{\sum_{(i,j)} N_{t_{\text{sw}}}(i,j)}\right)
= O(\tau K \sqrt{T}),
\end{equation}
where we used Cauchy-Schwarz and $\sum_{(i,j)} N_t(i,j) \leq 2T$.

Since $\tau = O(\sqrt{\log KT})$:
\begin{equation}
\mathcal{R}_{\text{pre}} \leq O(K \sqrt{T \log KT}).
\end{equation}

\textit{Post-switch regret ($t \geq t_{\text{sw}}$):}
After switching, the algorithm runs EXP3. By Lemma~\ref{lem:exp3_regret}:
\begin{equation}
\mathbb{E}[\mathcal{R}_{\text{post}}] \leq O(K^{1/3}(T - t_{\text{sw}})^{2/3} (\log K)^{1/3}) \leq O(K^{1/3} T^{2/3} (\log K)^{1/3}).
\end{equation}

\textbf{Case 2: No switch occurs ($t_{\text{sw}} > T$).}

If no switch occurs throughout $T$ rounds, the adversary's cumulative deviation 
on each pair is bounded by the detection threshold. The total regret from 
adversarial deviations is:
\begin{align}
R_T^{\text{adv}} &\leq \sum_{(i,j)} O(\tau(\sqrt{N_T(i,j)} + 1)) \nonumber \\
&\leq O(\tau K^2) + O\left(\tau K \sqrt{T}\right) \nonumber \\
&= O(K^2 \sqrt{\log KT} + K \sqrt{T \log KT}).
\end{align}

\textbf{Combined bound.}
In both cases, the total adversarial regret is bounded by the sum of the exploration overhead, the maximum undetected pre-switch deviation, and the post-switch EXP3 regret:
\begin{equation}
    \mathbb{E}[R_T^{\mathrm{B}}] \leq O(K^2 \log KT) + O(K \sqrt{T \log KT}) + O(K^{1/3} T^{2/3} (\log K)^{1/3}).
\end{equation}

For $T \geq K^2$, the initial exploration phase $O(K^2 \log KT)$ is strictly dominated by the pre-switch deviation term $O(K \sqrt{T \log KT})$. 

However, neither the pre-switch deviation nor the post-switch EXP3 regret strictly dominates the other across all time horizons. The pre-switch detection delay dominates for moderate horizons ($T < K^4 / \log^2 K$), while the EXP3 regret dominates for massive horizons ($T > K^4 / \log^2 K$). Therefore, the final rigorous bound retains both primary terms:
\begin{equation}
    \mathbb{E}[R_T^{\mathrm{B}}] \leq O\left(K \sqrt{T \log KT} + K^{1/3} T^{2/3} (\log K)^{1/3}\right). \qedhere
\end{equation}
\end{proof}

\section{Proofs for Extensions}
\label{app:extensions}

We provide the formal corollaries and proofs for the time-varying subset size extension discussed in Section~\ref{sec:extensions}.

\begin{corollary}[Time-Varying Subset Size: Condorcet Setting]
\label{cor:varying_m_condorcet}
Theorem~\ref{thm:main} and Corollary~\ref{cor:stochastic} hold when the subset size $m_t \in \{2,\ldots,K\}$ varies arbitrarily across rounds.
\end{corollary}

\begin{proof}
The proof of Lemma~\ref{lem:regret_equiv} holds for each round independently: at round $t$, the rescaling factor $\beta_{m_t}$ depends only on the current subset size $m_t$. The key observations are:
\begin{enumerate}[label=(\roman*), leftmargin=*, itemsep=1pt]
    \item The feedback $o_t$ remains an unbiased estimate of $\widehat{P}_t(x_t, y_t) = \alpha_{m_t} + \beta_{m_t} P_t(x_t, y_t)$.
    \item The base algorithm operates on the rescaled preferences and is agnostic to the specific value of $\beta_{m_t}$.
    \item The regret equivalence $\mathbb{E}[r_t^{\mathrm{C}}] = \frac{1}{\beta_{m_t}} \mathbb{E}[\hat{r}_t^{\mathrm{DB}}]$ holds pointwise.
\end{enumerate}
Since $\beta_{m_t} \geq 1/2$ for all $m_t \geq 2$, the bounds carry through with the same constants.
\end{proof}

\begin{corollary}[Time-Varying Subset Size: Borda Setting]
\label{cor:varying_m_borda}
Theorems~\ref{thm:stochastic_formal} and~\ref{thm:adversarial_formal} hold when $m_t \in \{2,\ldots,K\}$ varies arbitrarily across rounds, provided the transformed feedback $g_t$ from \textsc{MetaDueling} uses the current $m_t$.
\end{corollary}

\begin{proof}
The transformed feedback $g_t = (o_t - \alpha_{m_t})/\beta_{m_t}$ remains an unbiased estimate of $P_t(x_t, y_t)$ for any $m_t$. The stochastic analysis (exploration, deviation detection) and adversarial analysis (EXP3) are unaffected by varying $m_t$.
\end{proof}
\section{Additional Technical Results}
\label{app:additional}

\subsection{Concentration Inequalities}
\label{app:concentration}

We state the concentration inequalities used throughout the proofs.

\begin{lemma}[Hoeffding's Inequality]
\label{lem:hoeffding}
Let $X_1, \ldots, X_n$ be independent random variables with $X_i \in [a_i, b_i]$. 
Then for $\bar{X} = \frac{1}{n}\sum_{i=1}^n X_i$:
\begin{equation}
\Pr[|\bar{X} - \mathbb{E}[\bar{X}]| \geq t] \leq 2\exp\left(-\frac{2n^2 t^2}{\sum_{i=1}^n (b_i - a_i)^2}\right).
\end{equation}
\end{lemma}

\begin{lemma}[Freedman's Inequality]
\label{lem:freedman}
Let $(M_n)_{n \geq 0}$ be a martingale with $M_0 = 0$ and $|M_n - M_{n-1}| \leq c$ 
almost surely. Let $V_n = \sum_{k=1}^n \mathbb{E}[(M_k - M_{k-1})^2 \mid \mathcal{F}_{k-1}]$ 
be the predictable quadratic variation. Then for any $x, v > 0$:
\begin{equation}
\Pr[\exists n: M_n \geq x \text{ and } V_n \leq v] \leq \exp\left(-\frac{x^2}{2v + 2cx/3}\right).
\end{equation}
\end{lemma}

\end{document}